\documentclass{article}

    \PassOptionsToPackage{numbers, compress}{natbib}

\usepackage[preprint]{neurips_2022}

\usepackage[utf8]{inputenc} 
\usepackage[T1]{fontenc}    
\usepackage{hyperref}       
\usepackage{url}            
\usepackage{booktabs}       
\usepackage{amsfonts}       
\usepackage{nicefrac}       
\usepackage{microtype}      
\usepackage{xcolor}         

\usepackage[ruled,vlined,linesnumbered,noend]{algorithm2e}
\usepackage{enumitem}
\usepackage{graphicx}
\usepackage{amssymb}
\usepackage{mathrsfs}
\usepackage{amsmath}
\usepackage{mathtools}
\usepackage{subfigure}
\usepackage{graphicx}
\usepackage{wrapfig}
\usepackage{natbib}

\newtheorem{theorem}{Theorem}
\newtheorem{lemma}{Lemma}

\newtheorem{corollary}{Corollary}
\newtheorem{conjecture}{Conjecture}
\newtheorem{definition}{Definition}

\newenvironment{proof}{{\noindent\it Proof.}\quad}{\hfill $\square$ \par}
\newenvironment{proof_sketch}{{\noindent\it Proof Sketch.}\quad}{\hfill $\square$ \par}

\title{Towards Soft Fairness in Restless Multi-Armed Bandits}

\author{%
  Dexun Li \\
  Singapore Management University\\
  Singapore 178902\\
  \texttt{dexunli.2019@phdcs.smu.edu.sg} \\
   \And
   Pradeep Varakantham \\
   Singapore Management University \\
   Singapore 178902 \\
   \texttt{pradeepv@smu.edu.sg} \\
}

\begin{document}

\maketitle

\begin{abstract}
Restless multi-armed bandits (RMAB) is a framework for allocating limited resources under uncertainty. It is an extremely useful model for monitoring beneficiaries and executing timely interventions to ensure maximum benefit in public health settings (e.g., ensuring patients take medicines in tuberculosis settings, ensuring pregnant mothers listen to automated calls about good pregnancy practices). Due to the limited resources, typically certain communities or regions are starved of interventions that can have follow-on effects. To avoid starvation in the executed interventions across individuals/regions/communities, we first provide a soft fairness constraint and then provide an approach to enforce the soft fairness constraint in RMABs. The soft fairness constraint requires that an algorithm never probabilistically favor one arm over another if the long-term cumulative reward of choosing the latter arm is higher. Our approach incorporates \textit{softmax} based value iteration method in the RMAB setting to design selection algorithms that manage to satisfy the proposed fairness constraint. Our method, referred to as SoftFair, also provides theoretical performance guarantees and is asymptotically optimal. Finally, we demonstrate the utility of our approaches on simulated benchmarks and show that the soft fairness constraint can be handled without a significant sacrifice on value.

\end{abstract}

\section{Introduction}
Restless Multi-Armed Bandit(RMAB) Process is a generalization of the classical Multi-Armed
Bandit (MAB) Process, which has been studied since the 1930s~\cite{katehakis1987multi}. RMAB is a powerful framework for budget-constrained resource allocation tasks in which a decision-maker must select a subset of arms for interventions in each round. Each arm evolves according to an underlying Markov Decision Process (MDP). The overall objective in a RMAB model is to sequentially select arms so as to maximize the expected value of the cumulative rewards collected over all the arms. RMAB is of relevance in public health monitoring scenarios, recommendation systems and many others.
Tracking a patient's health or adherence and intervening at the right time is an ideal problem setting for an RMAB~\cite{akbarzadeh2019restless, bhattacharya2018restless, mate2020collapsing}, where the patient health/adherence state is represented using an arm. Resource limitation constraint in RMAB comes about due to the severely limited availability of healthcare personnel.  By developing practically relevant approaches for solving RMAB within severe resource limitations,  RMAB can assist patients in alleviating health issues such as diabetes~\cite{newman2018community}. hypertension~\cite{brownstein2007effectiveness}, tuberculosis~\cite{chang2013house, ong2014effects}, depression~\cite{lowe2004monitoring, mundorf2018reducing}, etc.

 While Whittle index based approaches~\cite{mate2021efficient,lee2019optimal,li2022efficient} address the RMAB problem with a finite time horizon by providing an asymptotically optimal solution, they are susceptible to starving arms, which can have severe repercussions in public health scenarios. Owing to the deterministic selection strategy of picking arms that provide the maximum benefit, in many problems, only a small set of arms typically get picked. As shown in our experimental analysis, Figure~\ref{fig:distribution} provides one example, where almost 50\% of the arms do not get any interventions using the Whittle index approach. While it is an optimal decision, it should be noted that interventions help educate patients or beneficiaries on potential benefits and starvation of such interventions for many patients can result in a lack of proper understanding of the program and reduce its effectiveness in the long run. Thus, there is a need to not starve arms without significantly sacrificing optimality.

Existing works have proposed different notions of fairness in the context of MAB to prevent starvation by enabling the selection of non-optimal arms.
\citet{li2019combinatorial} study a new Combinatorial Sleeping MAB model with Fairness constraints, called CSMAB-F. Their fairness definition requires algorithm to ensure a minimum selection fraction for each arm.
\citet{patil2020achieving} introduce similar fairness constraints in the stochastic MAB problem, where they use a pre-specified vector to denote the guaranteed number of pulls.
\citet{joseph2016fairness} define fairness as saying that a worse arm should not be picked compared to a better arm, despite the uncertainty on payoffs.
\citet{li2022efficient} form the allocation decision-making problem as the RMAB with fairness constraints, where fairness is defined as a minimum rate at which a
task or resource is assigned to a user.
Since knowing the guaranteed number (or proportion) of pulls is difficult to ascertain {\em a priori}, we generalize on these fairness notions for MAB and the fairness notion introduced by Jabbari {\em et al.}~\cite{jabbari2017fairness} for a reinforcement learning setting. We introduce a {\em soft fairness} constraint for RMABs, that requires that an RMAB algorithm never favor an arm probabilistically over another arm, if the long-term cumulative reward of choosing the latter arm is higher.

In summary, our goal is to compute stochastic policies for selecting arms in finite horizon RMAB, which satisfy the soft fairness constraint. To that end, we make the following contributions:
\begin{itemize}
    \item A practically relevant algorithm called \textit{SoftFair}, that enforces the soft fairness constraint and thereby avoids starvation of interventions for arms. Unlike the well-known Whittle index algorithm, \textit{SoftFair} does not require any indexability assumptions.
    \item Performance bounds and theoretical properties of the \textit{SoftFair} algorithm.
    \item Detailed experimental results which demonstrate that \textit{SoftFair} is competitive with other policies while satisfying the soft fairness constraint.
\end{itemize}

\section{RMAB with Soft Fairness Constraint}
\label{sec:prob}

In this section, we formally introduce the RMAB problem with the soft fairness constraint.
There are $N$ independent arms, each of which evolves according to an associated Markov Decision Process (MDP), characterized by the tuple $\{\mathcal{S}, \mathcal{A}, \mathcal{P}, \mathcal{R}, T, \gamma\}$. $\mathcal{S}$ represents the state space, $\mathcal{A}$ represents the action space, $\mathcal{P}$ represents the transition function, and $\mathcal{R}$ is the reward function that lies within the interval, $[R_{min}, R_{max}]$. $T$ is the horizon in each episode, and $\gamma$ is the discount factor. We use $s_i$ and $a_i$ to denote a state and action for arm $i$, respectively. Let $\mathbf{s}=\{s_1,\dots, s_n \}$ and $\mathbf{a}=\{a_1,\dots, a_n \}$ denote the state vector and action vector of RMAB over all arms, respectively.

A policy $\pi$ maps from the states to a distribution over actions. Particularly, $\pi(\mathbf{s},\mathbf{a})\in [0,1]$ denotes the probability of selecting action $\mathbf{a}$ in the state $\mathbf{s}$ for the RMAB, with $\sum_{\mathbf{a}}\pi(\mathbf{s},\mathbf{a})=1$.
Similar to~\citet{jabbari2017fairness}, we define the fairness using the state-action value function $Q^{\ast}(\mathbf{s},\mathbf{a})$ as follows:
\begin{definition}
(Fairness) A stochastic policy, $\pi$ is fair if for any time step $t\in[T]$, any joint state $\mathbf{s}$ and actions $\mathbf{a}, \mathbf{a}^\prime$, where $\mathbf{a} \neq \mathbf{a}^\prime$:
\begin{equation}\label{eq:fairness}
    \pi_t(\mathbf{s},\mathbf{a}) \geq \pi_t(\mathbf{s}, \mathbf{a}^\prime) \text{ only if } Q^\ast(\mathbf{s},\mathbf{a})\geq Q^\ast(\mathbf{s},\mathbf{a}^\prime)
\end{equation}
\end{definition}
\textbf{\textit{In summary, the goal of a solution approach is to generate a stochastic policy that never prefers one action over another if the cumulative long-term reward of selecting the latter one is higher.}}
The notations that are frequently used in this paper are summarized in Table~\ref{tab:notation}.

In this paper, we specifically consider a discrete RMAB where each arm has two states $\{ 0,1 \}$, and $1 (0)$ represents being in the “good” (“bad”) state, and there is a finite time horizon $T$ for each episode where $T$ is known to the algorithm in advance. At each time step $t\in[T]$, the algorithm can choose $k$ arms to pull according to observed states $\mathbf{s}=\{s_1,\dots,s_n\}$, and all arms undergo an action-dependent Markovian state transition process to a new state $\mathbf{s}^\prime = \{ s_1^\prime ,\dots,s_n^\prime \}$. Each arm independently receives a reward determined by its new state. Specifically, $a_{i,t}=1$ represents the choice to select arm $i\in[n]$ at time step $t\in[T]$ (active action), and $a_{i,t}=0$ represents the decision to be passive for arm $i$ (passive action). Then $\mathbf{s}_t\in\{0,1\}^n$ denotes the vector of states observed at time step $t$, and $\mathbf{a}_t\in \{ 0,1 \}^n$ denotes the vector of actions taken at $t$.  We have
\begin{equation}\label{eq:resource}
    ||\mathbf{a}_t||=k,
\end{equation}
and $ 1\leq k \ll n$ represents the limited resource constraint. $R_t(\mathbf{s}_t, \mathbf{a}_t)=\sum_{i=1}^n R_{i,t}(s_{i,t},a_{i,t})$ is the total reward obtained from RMAB at time step $t$ under the state $\mathbf{s}_t$ and action $\mathbf{a}_t$. We use a simple reward function: $R(s_i,a_i)=s_i^\prime \in\{0,1 \}$ determined by the next state $s_i^\prime$ obtained by taking action $a_i$ in the observed state is $s_i$ for all arms $i\in[n]$, note that the expected immediate reward should be $\mathbb{E}[R(s_i,a_i)]=P_{s_i,1}^{a_i}$.

\begin{table*}[ht]
\caption{Notations}
\resizebox{\linewidth}{!}{$
\begin{tabular}{|l|l|}
\toprule
\multicolumn{1}{|c|}{\textbf{Notation}} & \multicolumn{1}{c|}{\textbf{Description}}                                                                      \\ \toprule
$k,n,T$                                   & $n$:number of all competing arms in RMAB, $k$:number of arms can be selected each round, $T$: time horizon.                               \\ \midrule
$c$                                   & $c$: multiplier parameter.                               \\ \midrule
$s_i,a_i$, $\mathbf{s},\mathbf{a}$                                   & $s_i,a_i$: state and action of arm $i$, $\mathbf{s},\mathbf{a}$: state vector and action vector of RMAB.                               \\ \midrule
$[n],[T]$                               & We use {[}n{]} to represent the set of integers $\{0,\dots,n\}$ for $n\in \mathbb{N}$, so as {[}T{]}.          \\ \midrule
\begin{tabular}[c]{@{}l@{}}$Q_{m,t}(s,a)$,\\ $V_{m,t}(s)$\end{tabular}       &\begin{tabular}[c]{@{}l@{}} $Q_{m,t}(s,a)$: A state-action value function for the subsidy $m$ and state $s$ when taking action $a$ start at time step $t$\\ followed by optimal policy using Whittle index based approach in the future time steps; \\$V_{m,t}(s)$: Value function for the subsidy $m$ and state $s$ start at time step $t$ using Whittle index based approach\end{tabular} \\ \midrule
\begin{tabular}[c]{@{}l@{}}$Q_t(\mathbf{s},\mathbf{a})$,\\ $V_t(\mathbf{s})$\end{tabular} & \begin{tabular}[c]{@{}l@{}}$Q_t(\mathbf{s},\mathbf{a})$: The state-action value function when taking action $\mathbf{a}$ at time step $t$ with state $\mathbf{s}$\\ $V_t(s)$: The value function at the time step $t$ with state $\mathbf{s}$.
\end{tabular}                                                                                                           \\ \bottomrule
\end{tabular}
\label{tab:notation}
$}
\end{table*}


The objective of the algorithm $\pi$ is to efficiently approximate the maximum cumulative long-term reward while satisfying resource constraints and fairness constraints. Towards this end, the reward maximization problem can be formulated as
\begin{equation}
\begin{aligned}
        \underset{\pi}{\text{maximize }} \mathbb{E}_{\pi}[\sum_{t=1}^T  \gamma^{t-1}R_t(\mathbf{s}_t,\mathbf{a}_t|\mathbf{s}_0)] \\
        \text{ such that Equation.~\ref{eq:resource}, and Equation.~\ref{eq:fairness} are satisfied}
\end{aligned}
\end{equation}

In this paper, we consider the problem of interventions for patient adherence behaviors, and we assign same value to the adherence of a given arm/patient over time.

\section{Method}
In this section, we design a probabilistically fair selection algorithm by carefully integrating ideas from value iteration methods with RMAB setting to deal with our objective. We show that our method, called \textit{SoftFair}, is fair under proposed fairness constraints.
\textit{SoftFair} relies on the notion of known next-state distribution. A next-state distribution is defined to be known as the algorithm has the full knowledge of the transition probabilities.
\textit{SoftFair} requires particular care in computing the action probability distributions, and must restrict the set of such policies to balance the fair exploration and fair exploitation polices. Correctly formulating this restriction process to balance fairness and performance relies heavily on the observations about the relationship between fairness and performance.


In order to implement the value iteration methods in the RMAB setting,  \textit{SoftFair} first need to identify the estimated value function of the state of each arm $i\in [n]$ at each time step, and calculate the difference of state-action value function between the active and passive action. Then \textit{SoftFair} maps each arm $i$'s state to an state-specific probability distribution over actions, such that for each time step $t$, $\pi_{t}(\mathbf{s},\mathbf{a}) \geq \pi_{t}(\mathbf{s}, \mathbf{a}^\prime) \text{ only if } Q^\ast(\mathbf{s},\mathbf{a})\geq Q^\ast(\mathbf{s},\mathbf{a}^\prime)$. Providing such decision support with a fairness mindset
can promote acceptability~\cite{rajkomar2018ensuring, kelly2019key}. In the case of beneficiaries, we assume that an arm/patient might consider action/participation fair when participation of a certain patient (i.e., due to receiving an active action) resulted in a greater increase in expected time spent in a adherent state compared to non-participation (i.e., the passive action on the arm/patient).
Finally, We can sample the actions to get the next states $\mathbf{s}^\prime$.
It suffices to consider a single arm process due to strong decomposability of the RMAB, we now give the details on how to construct our \textit{SoftFair} algorithm so as to efficiently approximate our constrained long-term reward maximization objective.

The \textit{SoftFair} method first independently computes the logit value $\lambda_i$ based on the value function for each arm $i\in[n]$ under state $s_{i,t}$ at time step $t$ of episode $ep$, where $a_{i,t} \in\{ 0,1 \}$.
\begin{equation}\label{eq:Q(s,a)}
    \begin{aligned}
    \zeta_{i,t}^{ep}(s_{i,t},a_{i,t}) &= e^{{Q}_{i,t}^{ep}(s_{i,t},a_{i,t})-V_{i,t}^{ep}(s_{i,t})}\\
    Q_{i,t}^{ep}(s_{i,t},a_{i,t}) &= R_{i,t}(s_{i,t},a_{i,t}) + \gamma \underset{s_{i,t+1}^\prime}{\sum}\Pr(s_{i,t+1}^\prime|s_{i,t},a_{i,t})V^{ep}_{i,t+1}(s_{i,t+1}^\prime)\\
    \lambda_{i,t}^{ep} &=\log \zeta_{i,t}^{ep}(s_{i,t},a_{i,t}=1) - \log \zeta_{i,t}^{ep}(s_{i,t},a_{i,t}=0)\\
    \end{aligned}
\end{equation}
Here $V_{i,t}^{ep}(\cdot)$ is the value function of arm $i$ in the observed state of episode $ep$.
Then \textit{SoftFair} maps all arms $i\in[n]$ to probability distribution over actions set $\mathbf{a}_t=\{ a_{1,t}, \dots, a_{n,t}\}$ based on the observed states vector $\mathbf{s}_t= \{s_{1,t},\dots, s_{n,t} \}$. Note that here $||\mathbf{a}_t||=1$ means that only one arm will be selected. We can sample $k$ times without replacement to get $k$ arms to pull, which ensures that we meet the resource constraint as well as the fairness constraint, and then is apprised of the next state $\mathbf{s}^\prime$. More specifically, we employ the \textit{softmax} function to compute the corresponding action probability distribution.
\begin{equation}\label{eq:lambda}
\begin{aligned}
    &\pi^{ep}(\mathbf{s},\mathbf{a}=\mathbb{I}_{\{i\}}) = \textit{softmax}_c ( c\cdot \lambda_i^{ep})=\frac{\exp( c\cdot  \lambda_i^{ep})}{\sum_{i=1}^n \exp( c\cdot \lambda_i^{ep})}
\end{aligned}
\end{equation}
where $\mathbf{a}=\mathbb{I}_{\{i\}}$~\footnote{$\mathbb{I}_{\{i\}}$ is the indicator with value 1 at the $i$th term and value 0 at other places} denotes the action is to select arm $i$ while keeping other arms passive, and $\pi^{ep}(\mathbf{s},\mathbf{a}=\mathbb{I}_{\{i\}})$ denotes the probability that arm $i$ will be selected under state $\mathbf{s}$. $c\in(0,\infty)$ is the multiplier parameter~\footnote{The updation process of our \textit{Softfair} algorithm will converge to the Bellman Equation~\ref{eq:Psi} with an exponential rate in terms of $c$~\cite{song2019revisiting}, and $c$ controls the asymptotic performance~\cite{kozuno2019theoretical}.} that can adjust the gap between the probabilities of choosing an arm. When $c=\infty$, \textit{SoftFair} becomes the standard optimal Bellman operations~\cite{asadi2017alternative} (Refer to Equation~\ref{eq:Psi}). After computing the relative probability that the arm will be selected, we then can derive the probability that the arm $i$ is among the $k$ selected arms, denoted as $\Pr(a_{i,t}=1|\mathbf{s})$~\footnote{This can be
computed through the permutation iteration. Note $\Pr(a_i=1|\mathbf{s})=\textit{softmax}_c ( c\cdot \lambda_i)$ if $k=1$}.
For each arm $i$, the value function $V^{ep}(\cdot)$ update at episode $ep$ can be written as follows, for every $t\in[T]$:
\begin{equation}\label{eq:value}
\resizebox{\linewidth}{!}{$
  V^{ep}_{i,t}(s) =
    \begin{cases}
        \underset{a_{i,t}\in\{0,1\}}{\sum}\Pr(a_{i,t}|\mathbf{s}_t)Q^{ep}(s_{i,t},a_{i,t})=
        \underset{a_{i,t}\in\{0,1\}}{\sum}\Pr(a_{i,t}|\mathbf{s}_t)\underset{s^\prime_{i,t+1}\in \{0,1 \}}{\sum} \Pr(s^\prime_{i,t+1}|s_{i,t},a_{i,t})(R(s_{i,t},a_{i,t})+\gamma  V^{ep-1}_{i,t+1}(s^\prime_{i,t+1})) &\text{ if } (s=s_{i,t})\\
        V^{ep-1}_{i,t}(s) &\text{ otherwise}
        \end{cases}
        $}
\end{equation}
Similarly, we can also rewrite update equation for the state-action value function, we provide it in appendix.
The process of \textit{SoftFair} is summarized in Algorithm~\ref{al:softfair}.

\begin{algorithm}[ht]
\caption{SoftFair Value Iteration (\textit{SoftFair})}
\label{al:softfair}
\SetAlgoLined
\KwIn{Transition matrix $P$, time horizon $T$, set of observed states $\mathbf{s}$, resource constraint $k$,  multiplier parameter $c$, episodes $K$}
    $V_{i,t}(s) \leftarrow 0, \forall s, i, t$\;
\For{episode $ep=1,\dots,K$}{
Initialize $\mathbf{s}_0 = \{s_{1,0},\dots, s_{n,0}\}$

\For{step $t=0,\dots,T$}{
    \For{arm $i=1,\dots, n$}{
    Compute the $Q_{i,t}^{ep}(s_{i,t},a_{i,t})$ and $\lambda_{i,t}^{ep}(s_{i,t},a_{i,t})$ according to Equation.~\ref{eq:Q(s,a)}\;

    }
    Compute the probability $\pi^{ep}$ according to Equation.~\ref{eq:lambda}\;
    Sample $k$ arms and add them into action set\;
    \For{arm $i=1,\dots, n$}{
        Compute the probability that it will be activated $\Pr(a_i=1|\mathbf{s})$\;
        Update the value function $V_{i,t}^ep(s)$ according to Equation.~\ref{eq:value}
    }
    Play the arm in the action set, and observe next state $\mathbf{s}_{t+1}$
}

}
\KwOut{The value function $V_i(s)$ for arm $i\in [n]$}
\end{algorithm}

\section{Analysis of \textit{SoftFair}}

In this section we formally analyze \textit{SoftFair} and associated theoretical supports. We begin by connecting \textit{SoftFair} with the well-known Whittle index algorithm mentioned earlier, and show why the Whittle index approach is not suitable for our case (Fairness constraint and Finite horizon).

\subsection{\textit{SoftFair} \textbf{\textit{vs}}. Whittle index based methods}
Whittle index policy is known to be the asymtotically optimal solution to RMAB under the infinite time horizon. It independently assigns an index for each state of each arm to measure how attractive it is to pull an arm at a particular state.The index is computed using the concept of a "subsidy" $m$, which can be viewed as the opportunity cost of remaining passive, and is rewarded to the algorithm for each arm that is kept passive, in addition to the usual reward. Whittle index for an arm $i$ is defined as the infimum value of subsidy, $m$ that must be offered to the algorithm to make the algorithm indifferent between pulling and not pulling the arm. Consider a single arm $i\in[n]$ where the state is $s_i$.
At each time step $t\in[T]$, let $Q_{m;i,t}(s_i,a_i=0)$ and $Q_{m;i,t}(s_i,a_i=1)$ denote its active and passive state-action value functions under a subsidy $m$, respectively. We drop subscript $i$ when there is no ambiguity, i.e., $Q_{m;i,t}(s_i,a_i=0)=Q_{m,t}(s,a=0)$ and $Q_{m;i,t}(s_i,a_i=1)=Q_{m,t}(s,a=1)$. We can have:

\begin{equation}
\resizebox{\linewidth}{!}{$
\begin{aligned}
    Q_{m,t}(s,a=0) = P_{s,1}^0 + m +
    \gamma (((1-s)P_{0,0}^0+s P_{1,0}^0) V_{m,t+1}(0)+((1-s)P_{0,1}^0+s P_{1,1}^0 )V_{m,t+1}(1)) & \text{ passive}\\
    Q_{m,t}(s,a=1) = P_{s,1}^1 +
    \gamma (((1-s)P_{0,0}^1+s P_{1,0}^1) V_{m,t+1}(0)
    +((1-s)P_{0,1}^1+s P_{1,1}^1 )V_{m,t+1}(1)) & \text{ active}
\end{aligned}
$}
\end{equation}

The value function for the state $s$ is $V_{m,t}(s)=\max\{ Q_{m,t}(s,a=0), Q_{m,t}(s,a=1)\}$. The Whittle index $W(s)$ can be formally written as:
\begin{equation}
    W(s) = \underset{m}{\inf} \left\{ m_t:Q_{m,t}(s,a=0)= Q_{m,t}(s,a=1) \right\}.
\end{equation}
After computing the Whittle index for each arm, a policy $\pi$ will pull those $k$ arms whose current states have the highest indices at each time step. In order to use the Whittle index approach, it need to satisfy a technical condition called \textit{indexability} introduced by~\citet{weber1990index}. The indexability can be expressed in a simple way: Consider an arm with subsidy $m$, the optimal action is passive, then $\forall m^\prime > m$, the optimal action should remain passive. The RMAB is indexable if every arm is indexable.

However, in  the  case  of  interventions  with  regards  to public health,  the Whittle index approach concentrates on  the  beneficiaries who can mostly improve the objective (public health outcomes). This can lead to some beneficiaries never have a chance to get intervention from public health workers, resulting in a bad adherence behavior and henceforth a bad state from where improvements can be small even with intervention and thus never getting selected by the index policy. Refer to Figure~\ref{fig:distribution} to get a better picture of the difference between the Whittle index approach and \textit{SoftFair}. We can see that when using the Threshold Whittle index method proposed by~\citet{mate2020collapsing}, the activation  frequency of the arm is extremely unbalanced, with nearly half  of the  arms  never being selected. Such  starvation  of  interventions may escalate to communities. To avoid such cycle between bad outcomes, the RMAB needs to consider fairness in addition to maximizing cumulative long-term reward when picking arms.

\begin{figure}[ht]
\centering
    \includegraphics[width=0.8\linewidth]{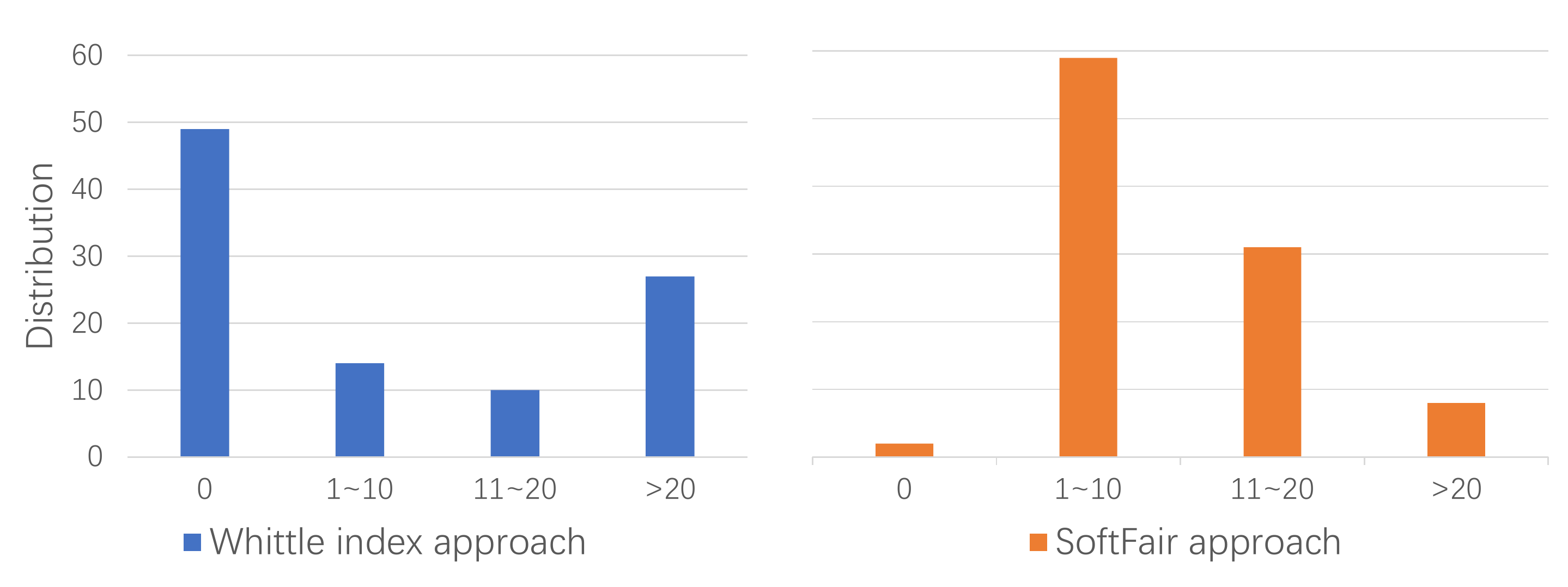}
    \caption{The x-axis is the number of times being selected, and the y-axis is the frequency distribution. We consider the RMAB given in Section~\ref{sec:prob}, with $k=10$, $n=100$, $T=100$. Left: the Whittle index algorithm. Right: \textit{SoftFair} ($c=2$). As can be noted, without fairness constraints in place, the arm activation frequency is lopsided, and almost 50\% of the arms never get activated. }
    \label{fig:distribution}
\end{figure}

However, in addition to failing to meet the system fairness requirements in many real-world applications, traditional Whittle index based approaches also rely on the assumption of an infinite time horizon, and the performance deteriorates severely when time horizons are finite. Often, real-world phenomena are formalized in a finite horizon setting, which prohibits direct use of Whittle index based methods. We now show why the Whittle index based approach can not be applied to the finite time horizon setting. We demonstrate that a phenomenon called Whittle index decay~\cite{mate2021efficient} exists in our problem. All detailed proofs can be found in the Appendix.



\begin{theorem}\label{thm:index_decay}
In the round $t\in[T]$, the Whittle index $m_t$ for arm $i$ under the observed state $s_i$ is the value that satisfies the equation $Q_{m,t}(s_{i,t},a_{i,t}=0)= Q_{m,t}(s_{i,t},a_{i,t}=1)$. The Whittle index will decay as the value of current time step $t$ increases: $\forall t <T: m_t>m_{t+1}\geq m_T = P_{s,1}^1-P_{s,1}^0$.
\end{theorem}
\begin{proof_sketch}
We first show a lemma to show value function $V_{m, t}(s_{i,t})>m_{i,t+1}(s_{i,t})\geq 0$, and then we can calculate $m_T$ and $m_{T-1}$ by solving equation $Q_{m,T}(s_{i,T},a_{i,T}=0)=Q_{m,T}(s_{i,T},a_{i,T}=1)$ and $Q_{m,T-1}(s_{i,T-1},a_{i,T-1}=0)=Q_{m,T-1}(s_{i,T-1},a_{i,T-1}=1)$. We then can derive $m_t>m_{t+1}$ by obtaining $\frac{\partial m_t}{\partial t}<0$ for $\forall t>1$ based on the derived lemma. The detailed proof can be found in the appendix.
\end{proof_sketch}

The Whittle index based approach needs to solve the costly finite horizon problem because the index value varies according to the time step even in the same state, and computing the index value under the finite horizon setting is ($O(|S|^kT)$ time and space complexity~\cite{hu2017asymptotically}. However, as an alternative method, our \textit{SoftFair} can naturally approximate the optimal value function at arbitrary time steps while requiring less memory space than model-free learning methods such as Q-learning. We now demonstrate why \textit{SoftFair} can satisfy our proposed fairness constraint while effectively approximating our cumulative reward maximization objective.

\begin{theorem}\label{thm:lambda}
Choose top $k$ arms according to the $\lambda$ value in Equation~\ref{eq:lambda} ($c\rightarrow \infty$) is equivalent to maximize the cumulative long-term reward.
\end{theorem}
\begin{proof_sketch}
We first get the expression of $\sum_{i \in \phi}\lambda_i$ where $\phi$ is the set of selected arms, then we prove that getting the highest value $\sum_{i \in \phi}\lambda_i$ is equivalent to optimal policy.
\end{proof_sketch}

When $c$ approaches infinity, the algorithm becomes the optimal policy which will suffer from the starvation phenomena. Given these facts, $c$ can control the trade-off between the optimal performance and the fairness constraint.

\begin{theorem}\label{thm:soft_fair}
\textit{SoftFair} is fair under our proposed fairness constraint, and $c$ controls the trade-off between fairness and optimal performance.
\end{theorem}
\begin{proof_sketch}
The trade-off is governed by $c$, where a large $c$ means that \textit{SoftFair} tends to choose arms with higher value, while a small $c$ means that \textit{SoftFair} tends to ensure fairness among arms.
\end{proof_sketch}

\subsection{Performance bound of \textit{SoftFair}}




We investigate $k=1$ case, since the multi-selection at each time step can be viewed as an iteration. Let $\Psi_{soft}$ denote our \textit{Soft} operator at time step $t\in[T]$, we ignore the subscript $t$ here, which is
\begin{equation}\label{eq:Psi_soft}
\begin{aligned}
     Q^{ep+1}(\mathbf{s},\mathbf{a})&= \Psi_{soft}Q^{ep}(\mathbf{s},\mathbf{a})= \underset{\mathbf{s}^\prime}{\sum}  \Pr(\mathbf{s}^\prime|\mathbf{s},\mathbf{a})(R(\mathbf{s},\mathbf{a})+\gamma  \underset{\mathbf{a}^\prime}{\sum} \Pr(\mathbf{a}^\prime|\mathbf{s}^\prime) Q^{ep}(\mathbf{s}^\prime,\mathbf{a}^\prime))\\
     &= R(\mathbf{s},\mathbf{a})+\gamma \underset{\mathbf{s}^\prime}{\sum}  \Pr(\mathbf{s}^\prime|\mathbf{s},\mathbf{a}) \underset{\mathbf{a}^\prime}{\sum} \Pr(\mathbf{a}^\prime|\mathbf{s}^\prime) Q^{ep}(\mathbf{s}^\prime,\mathbf{a}^\prime)
     \end{aligned}
\end{equation}

Before we derive the performance bound for \textit{SoftFair}, We can first show how to bound the state-action value function in the following lemma.
\begin{lemma}\label{lem:lem}
The state-action value function $Q(\mathbf{s},\mathbf{a})$ is bounded within $[0,\frac{n}{1-\gamma}]$.
\end{lemma}
\begin{proof_sketch}
The upper bound can be obtained by showing that $\forall (\mathbf{s},\mathbf{a})$, state-action value at the $ep-$th iteration are bounded through induction.
\end{proof_sketch}

\begin{corollary}
As we have $R_{max} =n$ and $R_{min}=0$, we can easily derive that $|Q(\mathbf{s},\mathbf{a})-Q(\mathbf{s},\mathbf{a}^\prime)|\leq \frac{n}{1-\gamma}$, for $\forall Q$ and $\forall \mathbf{s}$.
\end{corollary}

Follow~\citet{song2019revisiting}, we let $\delta(\mathbf{s})=\sup_Q \max_{\mathbf{a},\mathbf{a}^\prime}|Q(\mathbf{s},\mathbf{a})-Q(\mathbf{s},\mathbf{a}^\prime)|$ denote the largest distance between state-action value functions. Then we have the following lemma.


\begin{lemma}\label{lem:bound}
$\forall Q$ and $\forall \mathbf{s}$, Let $\Pr(\cdot|\mathbf{s}) = [\Pr(\mathbf{a}=\mathbb{I}_{\{1\}}|\mathbf{s}), \dots,\Pr(\mathbf{a}=\mathbb{I}_{\{n\}}|\mathbf{s})]^\top$ and $Q(\mathbf{s},\cdot) = [{Q(\mathbf{s},\mathbf{a}=\mathbb{I}_{\{1\}}}), \dots,{Q(\mathbf{s},\mathbf{a}=\mathbb{I}_{\{n\}}})]^\top$, here the superscript $\top$ denotes the vector transpose. We have $\frac{\delta(\mathbf{s})}{n\exp[c \cdot \delta(\mathbf{s})]}\leq\underset{\mathbf{a}}{\max} Q(\mathbf{s},\mathbf{a}) - (\Pr(\cdot|\mathbf{s}))^\top Q(\mathbf{s},\cdot)\leq\frac{n-1}{2+c}$.
\end{lemma}
\begin{proof_sketch}
We first need to replace $\Pr(\cdot|\mathbf{s})$ with $Q(\mathbf{s},\cdot)$, and then demonstrate it by looking at possible values for the difference between state-action value function with different actions.
\end{proof_sketch}

Different from {\em Soft} Operator in Eq.~\ref{eq:Psi_soft} ,let $\Psi$ denote the Bellman optimality operator, which we have
\begin{equation}\label{eq:Psi}
    Q^{ep+1}(\mathbf{s},\mathbf{a}) = \Psi Q^{ep}(\mathbf{s},\mathbf{a}) = R(\mathbf{s},\mathbf{a})+\gamma \sum_{\mathbf{s}^\prime}\Pr(\mathbf{s}^\prime|\mathbf{s},\mathbf{a})\underset{\mathbf{a}^\prime}{\max}Q^{ep}(\mathbf{s}^\prime,\mathbf{a}^\prime)
\end{equation}
For the optimal state-action value function, we have $\Psi Q^{\ast}(\mathbf{s},\mathbf{a})=Q^{\ast}(\mathbf{s},\mathbf{a})$.

\begin{theorem} \label{thm:soft_bound}
Our \textit{SoftFair} method can achieve the performance bound as
$\underset{ep\rightarrow \infty}{\lim \sup}V^{ep}(\mathbf{s})\leq V^{\ast}(\mathbf{s})$, where $V^{\ast}(\mathbf{s})$ is the optimal value function. More specifically, we have
\begin{equation*}
\begin{aligned}
\underset{ep\rightarrow \infty}{\lim \sup}Q^{ep}(\mathbf{s},\mathbf{a}) \leq Q^{\ast}(\mathbf{s},\mathbf{a}) \quad \text{and}\\
\underset{ep\rightarrow \infty}{\lim \inf}Q^{ep}(\mathbf{s},\mathbf{a}) \geq Q^{\ast}(\mathbf{s},\mathbf{a})- \frac{n-1}{(2+c)(1-\gamma)}
\end{aligned}
\end{equation*}
\end{theorem}
\begin{proof}
We prove the bound by induction based on Lemma~\ref{lem:lem} and \ref{lem:bound}.
\end{proof}

\begin{conjecture}
For the cause when multiple arms can be pulled at each time step, i.e., $k\ne 1$, Our \textit{SoftFair} method can achieve the performance bound as
$\underset{ep\rightarrow \infty}{\lim \sup}{\Psi^{ep}}V^{0}(\mathbf{s})\leq V^{\ast}(\mathbf{s})$. More specifically, we have
\begin{equation*}
\begin{aligned}
\underset{ep\rightarrow \infty}{\lim \sup}Q^{ep}(\mathbf{s},\mathbf{a}) = \underset{ep\rightarrow \infty}{\lim \sup}{\Psi^{ep}}Q^{0}(\mathbf{s},\mathbf{a
})\leq Q^{\ast}(\mathbf{s},\mathbf{a}) \quad \text{and}\\
\underset{ep\rightarrow \infty}{\lim \inf}Q^{ep}(\mathbf{s},\mathbf{a}) \geq Q^{\ast}(\mathbf{s},\mathbf{a})- \frac{n-k}{(2+c)(1-\gamma)}
\end{aligned}
\end{equation*}
\end{conjecture}

\begin{figure*}[t]
    \centering
    \subfigure[ ]{
    \includegraphics[width=0.31\linewidth]{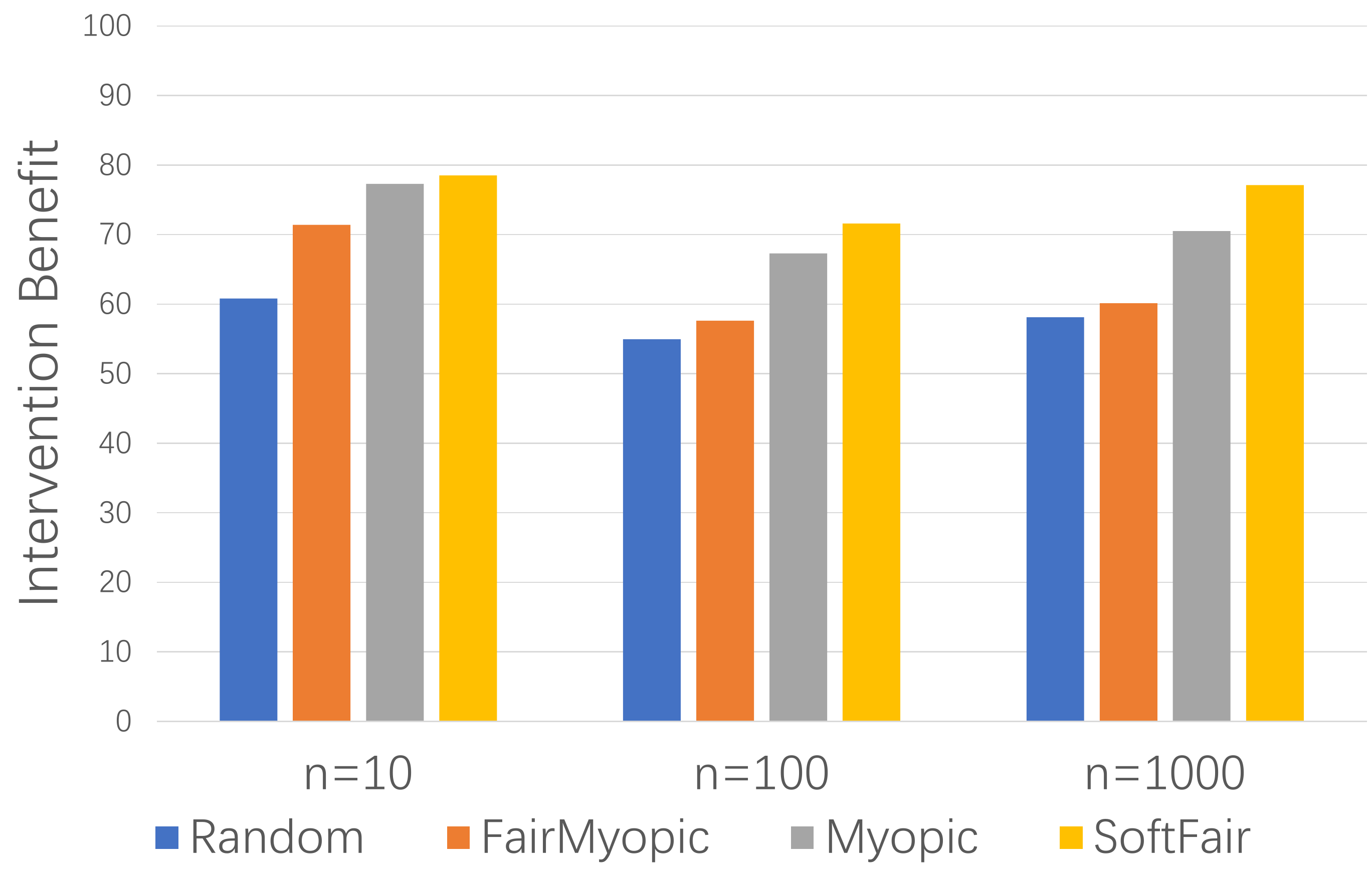}
    }
    \subfigure[ ]{
	\includegraphics[width=0.31\linewidth]{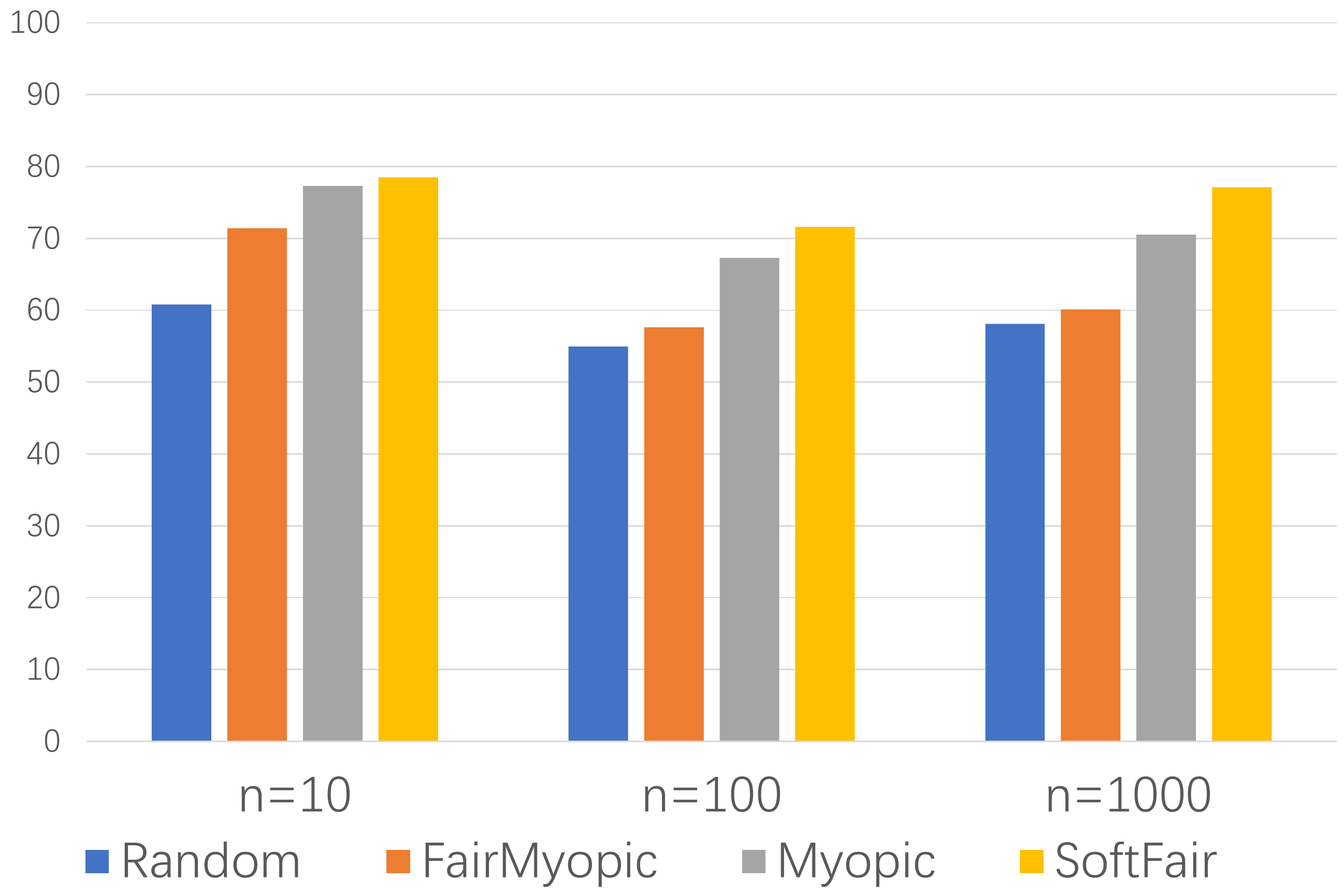}
    }
    \subfigure[ ]{
	\includegraphics[width=0.31\linewidth]{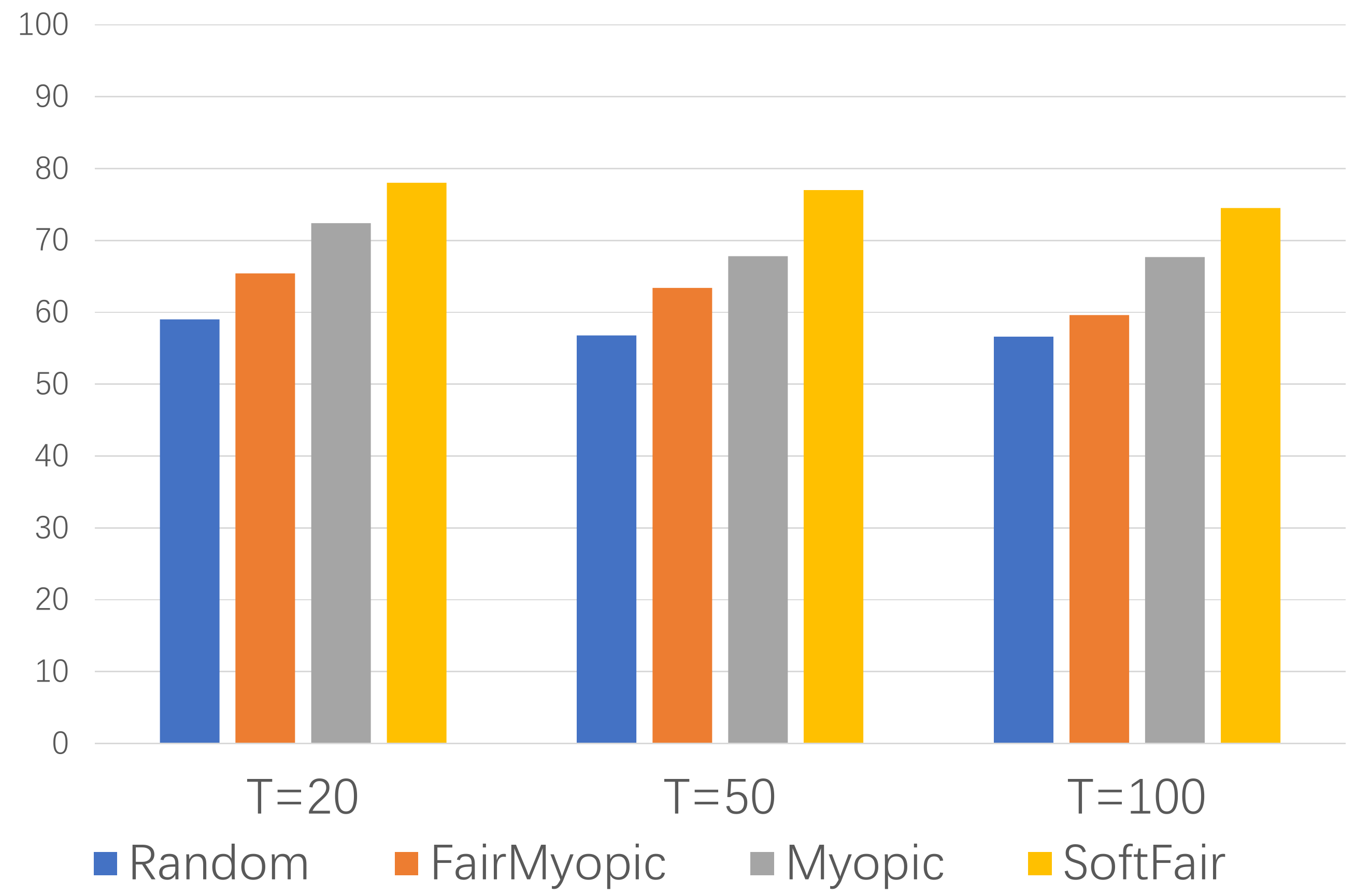}
    }
    \caption{Intervention benefit of \textit{SoftFair} is consistently greater than other baselines. (a) We fix $T=50$, and $k=10\%n$, and let $n=\{ 10,100,1000\}$. (b) We fix $T=50$, and $n=100$, and let $k=\{5,10,20\}$. (c) We fix $n=100$, and $k=10$, and let $T=\{ 20, 50, 100\}$.}
    \label{fig:Result}
\end{figure*}




\begin{figure*}[th]
    \centering
    \subfigure[ ]{
    \includegraphics[width=0.47\linewidth, height=1.5in]{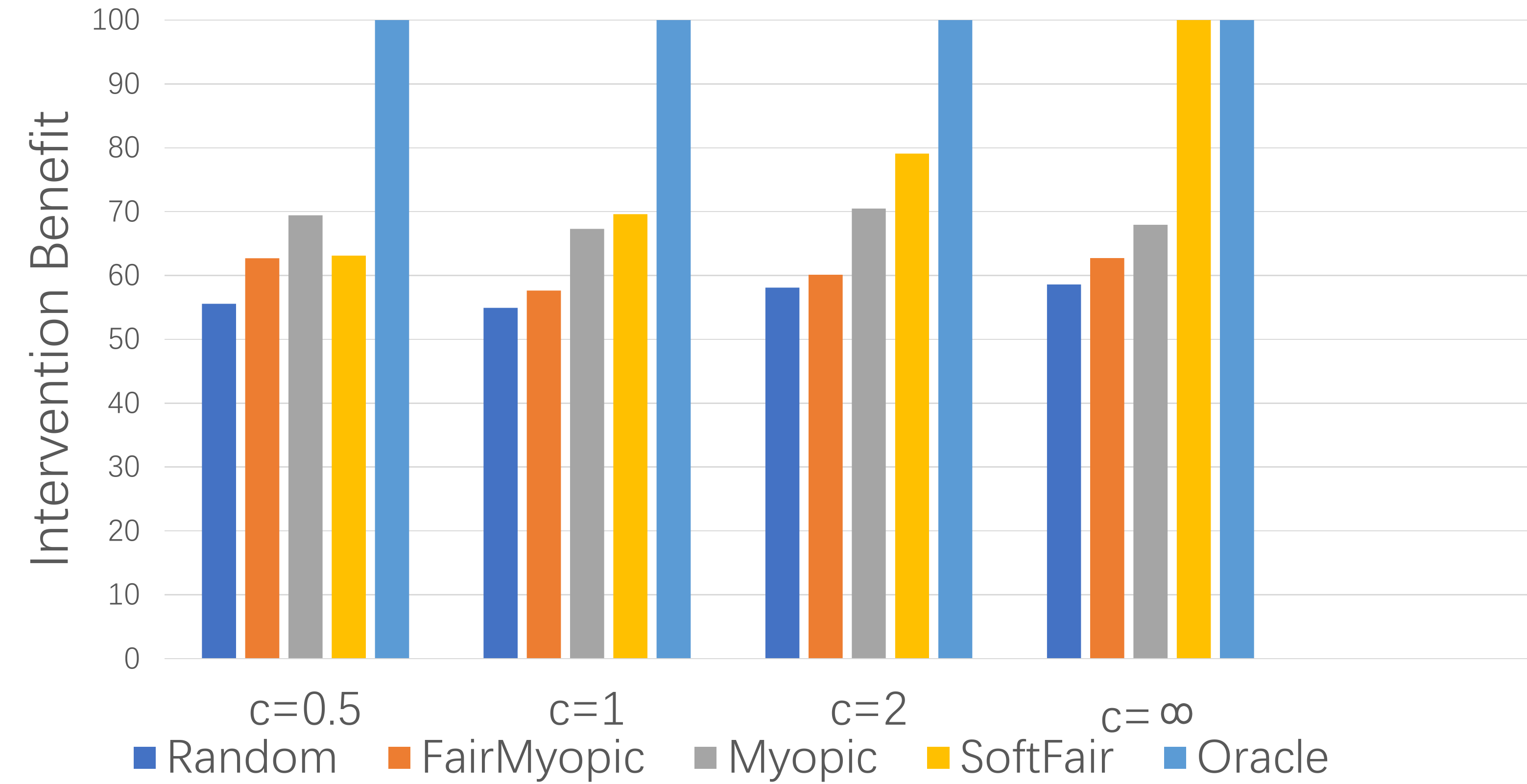}
    }\label{fig:trick}
    \subfigure[ ]{
	\includegraphics[width=0.47\linewidth,height=1.6in]{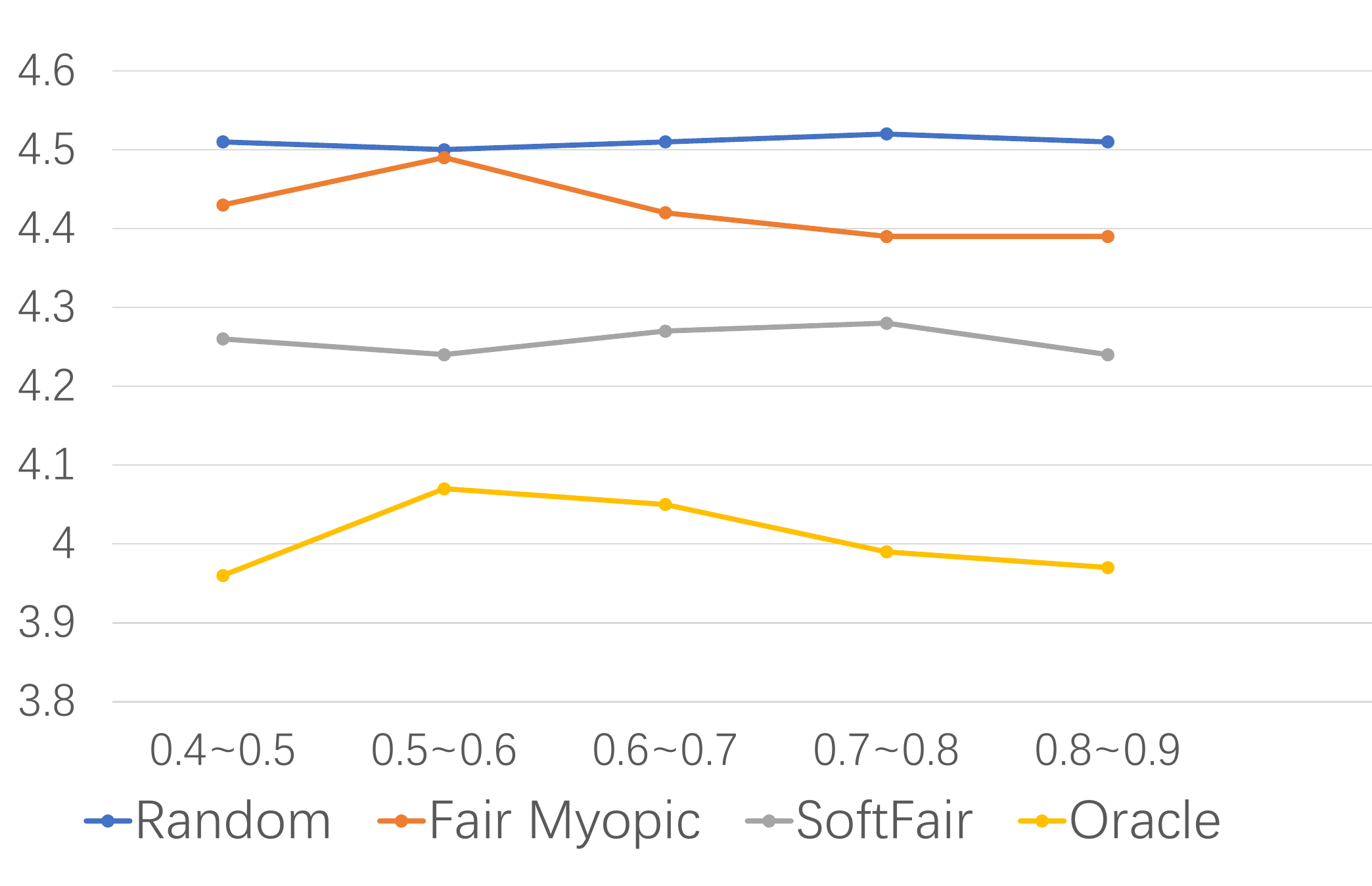}
    }\label{fig:entropy}
    \caption{(a) The intervention benefit of different multiplier $c$. Here $c=\infty$ refers to deterministically selecting the top $k$ arm with the highest cumulative rewards, but updating the value function with $c=1$. (b) The action entropy of a single process. We investigate the action entropy for different value of $P_{0,1}^1$ range from $0.4$ to $0.9$.}
\end{figure*}


\begin{table}[ht]
\centering
\caption{Resutls for CAPA Adherence dataset with $n=100,k=10,T=80$.}
\begin{tabular}{|l|c|c|}
\hline
Policy      & Intervention benefit & Action entropy   \\ \toprule
Random      & $79\pm13$            & $4.56\pm 0.0056$ \\ \midrule
Myopic      & $98\pm 3.3$          & $2.67\pm 0.0$   \\ \midrule
FairMyopic & $83\pm 11$           & $4.5\pm 0.0089$  \\ \midrule
\textit{SoftFair}    & $93\pm 7.6$          & $4.27\pm 0.019$  \\ \bottomrule
\end{tabular}
\label{tab:result}
\end{table}

\section{Experiments}
In this section, we empirically demonstrate that our proposed method \textit{SoftFair} that enforces the probabilistic fairness constraint introduced in Section~\ref{sec:prob}, can effectively approximate the cumulative reward maximization objective compared to the baselines on both (a) a realistic patient adherence behavior dataset~\cite{kang2013markov} and (b) a synthetic dataset that the underlying structural constraints outlined in Appendix~\ref{app_sec:dataset} are preserved. We consider the average rewards criterion over the finite time horizon where we use the discount factor $\gamma=1$, and set the following scenario for the simulation: $n=\{ 10, 100, 1000\}$, $k=\{5\%n, 10\%n, 20\%n\}$, $T=\{20, 50, 100\}$. All results are averaged over 50 simulations. In particular, We compare our method against the following baselines:

\begin{itemize}[leftmargin=*]
    \item \textbf{Random}: At each time step, algorithm randomly select $k$ arms to play. This will ensure that each arm has the same probability of being selected.
    \item \textbf{Myopic}: A myopic policy ignores the impact of present actions on future rewards and instead focuses entirely on the predicted immediate returns. It select $k$ arms that maximize the expected reward at the immediate next time step.
    Formally, this could be described as choosing the $k$ arms with the largest gap $\Delta_t = P_{s,1}^1-P_{s,1}^0$ at time step $t$ under the observed state $s$.
    \item \textbf{FairMyopic}: After computing $\Delta_t$ for each arm, instead of deterministically selecting the arm with the highest immediate reward, we use the \textit{softmax} function over $\Delta_t$ to get the probability of each arm being selected. Then we sample the $k$ arms according to the probability.
    \item \textbf{Oracle}: Algorithm by~\citet{mate2021efficient} under assumption that the states of all arms are fully observable and the transition probabilities are known without considering fairness constraints. We use a sigmoid function as~\citet{mate2021efficient} to approximate the Whittle index to select arms deterministically for the finite time horizon problem.
\end{itemize}


We examine policy performances from two perspectives: cumulative rewards and fairness. To this end, we rely on two performance metrics:(a) Intervention benefit: The intervention benefit is defined as $\frac{\bar{R}_{\text{method}}-\bar{R}_{\text{No intervention}}}{\bar{R}_{\text{Oracle}}-\bar{R}_{\text{No intervention}}}\times 100\%$. It calculates the difference between one algorithm's total expected cumulative reward and the total reward when no intervention is involved, then normalized by the difference between the reward obtained without intervention (0\% intervention benefit) and the asymptotically optimal but fairness-agnostic Oracle algorithm in baselines (100\% intervention benefit). (b) Action entropy: We calculate the selection frequency distribution across all time steps, and then compute its entropy after normalization through: $Entropy = -\sum_{i\in[n]}P(i)\log P(i)$, where $P(i)$ is the normalization of the number of times arm $i$ is selected (i.e., the number of times that arm $i$ has been selected divided by $k\cdot n$), and $P(i)\log P(i)=0$ if an arm is never selected across all time steps.

\paragraph{Realistic dataset}: Obstructive sleep apnea is one of the most prevalent sleep disorder among adults, and continuous positive airway pressure therapy (CPAP) is a highly effective treatment when it is used consistently for the duration of each sleep bout. But non-adherence to CPAP in patients hinders effective treatment of this type of sleep disorder. Similar to~\cite{herlihy2021planning}, we adapt the Markov model of CPAP adherence behavior in~\cite{kang2013markov} to a two-state system with the clinical adherence criteria. We add a small noise to each transition matrix so that the dynamics of each individual arm is different (See more details about the dataset in Appendix~\ref{app_sec:dataset}).

In table~\ref{tab:result}, we report average results for each algorithm. Myopic method has the best performance, which is caused by the specific structure of the underlying transition matrices, since there is not too much difference between $n$ Markovian models, and in this case the Myopic approach is indeed close to optimal. However, given our notion of fairness, the Myopic technique is not fair. Furthermore, Myopic policy might fail in some circumstances, and is even worse than Random policy~\cite{mate2020collapsing}. Meanwhile, our \textit{SoftFair} performs well while adhering to the specified fairness requirement.

\paragraph{Synthetic dataset}: (a) We first test the performance when the number of patients (arms) varies. Figure~\ref{fig:Result}a compares the
intervention benefit for $n = \{10, 100, 1000\}$ patients and $k = 10\%$ of $n$. As shown in Figure~\ref{fig:Result}a, in addition to satisfying the fairness constraints, our \textit{SoftFair} consistently outperforms the Random, Myopic and FairMyopic baselines. (b) We next compare the intervention benefit when the number of arms $n$ is fixed and the resource constraint $k$ is varied. Specifically, we fix $n=100$ patients, and let $k=\{ 5\%n, 10\%n,20\%n \}$. Figure~\ref{fig:Result}b shows that there has been a gradual increase in the intervention benefit as the $k$ increases. One possible reason is that a larger resource budget $k$ can make the arms with higher cumulative rewards more likely to be selected, thereby reducing the performance gap with the Oracle method. (c) The performance of our method is slightly influenced by the time horizon $T$. As shown in Figure~\ref{fig:Result}c, the common trend is that a smaller $T$ leads to better performance. This means that our method can efficiently solve the RMAB in a finite time horizon, while a larger horizon $T$ will make the convergence slower. Overall, all results demonstrate the effectiveness of our method compared to other baselines.




\paragraph{Intervention benefit when $c$ changes}: We also look into the effect of the multiplier parameter $c$ on performance. Formally, a larger $c$ will widen the gap between the probabilities of choosing an arm, resulting in a better performance as it prefers to choose an arm with a higher cumulative reward. Figure~\ref{fig:trick} reveals that \textit{SoftFair} performs well empirically as $c$ increases, and if we deterministically choose the top $k$ arms according to the $\lambda$ value, it will achieve the optimal result.

\paragraph{Action entropy comparison}: We also compare the entropy of the action of a process in the synthetic dataset when $P_{0,1}^1$ ranges from $0.4$ to $0.9$. As shown in Figure~\ref{fig:entropy}, the Random policy has the highest value as it requires uniform selection of all arms. Our proposed method, \textit{SoftFair} consistently has a higher action entropy than the Oracle method because we enforce fairness constraints. FairMyopic has a high action entropy value, but it is indeed unfair under our proposed fairness constraints, as it relies on immediate rewards.

\section{Conclusion}
In this paper, we study fairness constraints in the context of Restless Multi-Arm Bandit model, which is of critical importance for adherence problems in public health (e.g., monitoring adherence of prevention medicine for Tuberculosis, monitoring engagement of mothers during calls on good practices during pregnancy). To tackle the challenges introduced by the objective, we design a computationally efficient algorithm, a novel solution is proposed by integrating the \textit{softmax} value iteration technique in the RMAB setting. Our algorithms can effectively approximate the optimal value function within the proven performance bounds.

\bibliography{neurips_2022}
\bibliographystyle{plainnat}

\section*{Checklist}

\begin{enumerate}

\item For all authors...
\begin{enumerate}
  \item Do the main claims made in the abstract and introduction accurately reflect the paper's contributions and scope?
    \answerYes{} We incorporates \textit{softmax} based value iteration method in the RMAB setting to design selection algorithms that manage to satisfy the proposed fairness constraint. The Method, Analysis and Experiment section reflect these claims.
  \item Did you describe the limitations of your work?
    \answerNo{} There is the trade-off between the optimal performance and fairness constraint.
  \item Did you discuss any potential negative societal impacts of your work?
    \answerNo{} We aim to ensure the fairness in resource allocation problem.
  \item Have you read the ethics review guidelines and ensured that your paper conforms to them?
    \answerYes{}
\end{enumerate}

\item If you are including theoretical results...
\begin{enumerate}
  \item Did you state the full set of assumptions of all theoretical results?
    \answerYes{} Please see the Analysis Section and Appendix.
        \item Did you include complete proofs of all theoretical results?
    \answerYes{} Please see the Analysis Section and Appendix.
\end{enumerate}

\item If you ran experiments...
\begin{enumerate}
  \item Did you include the code, data, and instructions needed to reproduce the main experimental results (either in the supplemental material or as a URL)?
    \answerYes{} We include the code, datasets, readme file for reproducibility.
  \item Did you specify all the training details (e.g., data splits, hyperparameters, how they were chosen)?
    \answerYes{} Please see the Experiment Section.
        \item Did you report error bars (e.g., with respect to the random seed after running experiments multiple times)?
    \answerYes{} Please see Table 2.
        \item Did you include the total amount of compute and the type of resources used (e.g., type of GPUs, internal cluster, or cloud provider)?
    \answerNo{} It does not require high computing resources.
\end{enumerate}

\item If you are using existing assets (e.g., code, data, models) or curating/releasing new assets...
\begin{enumerate}
  \item If your work uses existing assets, did you cite the creators?
    \answerYes{} We correctly cite the papers that provide datasets and code of baselines for experiment.
  \item Did you mention the license of the assets?
    \answerNA{}
  \item Did you include any new assets either in the supplemental material or as a URL?
    \answerNA{}
  \item Did you discuss whether and how consent was obtained from people whose data you're using/curating?
    \answerNA{}
  \item Did you discuss whether the data you are using/curating contains personally identifiable information or offensive content?
    \answerNA{}
\end{enumerate}

\item If you used crowdsourcing or conducted research with human subjects...
\begin{enumerate}
  \item Did you include the full text of instructions given to participants and screenshots, if applicable?
    \answerNA{}
  \item Did you describe any potential participant risks, with links to Institutional Review Board (IRB) approvals, if applicable?
    \answerNA{}
  \item Did you include the estimated hourly wage paid to participants and the total amount spent on participant compensation?
    \answerNA{}
\end{enumerate}

\end{enumerate}


\newpage
\appendix

\section{Appendix}
\subsection{More Details about \textit{SoftFair}}
Similarly, we can also rewrite update equation for the state-action value function, we provide it in appendix.
\begin{equation}\label{eq:value2}
\resizebox{\linewidth}{!}{$
  Q^{ep}_{i,t}(s,a) =
    \begin{cases}
        R(s,a)+\underset{s^\prime_{i,t+1}}{\sum} \Pr(s^\prime_{i,t+1}|s,a)(\gamma  \underset{a^\prime}{\sum}\Pr(a^\prime|\mathbf{s}^\prime_{t+1})Q^{ep-1}_{i,t+1}(s^\prime_{i,t+1},a^\prime) &\text{ if } (s,a)=(s_{i,t},a_{i,t})\\
        Q^{ep-1}_{i,t}(s,a) &\text{ otherwise}
        \end{cases}
        $}
\end{equation}

The probability of choosing an arm is the \textit{softmax} function on $\lambda$.
We can write down the probability of $\pi(\mathbf{s},\mathbf{a}=\mathbb{I}_{\{i\}})$, where $i$ is the selected arms when $k= 1$. More specifically, we have $\pi(\mathbf{s},\mathbf{a}=\mathbb{I}_{\{i\}})= softmax_c(c\cdot \lambda_i)$, and note that $\Pr(a_i=1|\mathbf{s}) = softmax_c(c\cdot \lambda_i)$ denote that probability that arm $i$ is in the set of selected arms when $k=1$.

When $k\neq 1$, let the $\phi$ denote the set of selected arms, and $\mathbb{I}_{\{\phi\}}$ denote the action to select arms in set $\phi$ while keeping other arms passive. We have $\pi(\mathbf{s},\mathbf{a}=\mathbb{I}_{\{\phi\}}) = \Pi_{i\in \phi}\Pr(a_i=1|\mathbf{s})$, where $\Pr(a_i=1|\mathbf{s})$ can be obtained through the brute-force permutation iteration over $\pi(\mathbf{s},\mathbf{a}=\mathbb{I}_{\{i\}})= softmax_c(c\cdot \lambda_i)$.

\begin{definition}
(Fairness) Equivalently, a stochastic policy, $\pi$ is fair if for any time step $t\in[T]$, any joint state $\mathbf{s}$ and any two arms $i,j$, where $i \neq j$, The following two statements are equal:
\begin{equation}\label{eq:fairness2}
\begin{aligned}
    \pi_t(\mathbf{s},\mathbf{a}) \geq \pi_t(\mathbf{s}, \mathbf{a}^\prime) \text{ if and only if } Q^\ast(\mathbf{s},\mathbf{a})\geq Q^\ast(\mathbf{s},\mathbf{a}^\prime)\\
    \lambda_i \geq \lambda_j \text{ if and only if } \sum_{\phi_i: \phi_i \in \Phi_i}\pi_t(\mathbf{s},\mathbf{a} = \mathbb{I}_{\{ \phi_i \}}) \geq \sum_{\phi_j: \phi_j \in \Phi_j}\pi_t(\mathbf{s},\mathbf{a} = \mathbb{I}_{\{ \phi_j \}})
\end{aligned}
\end{equation}
Here $\Phi_i = \{ \phi_i \}$ and $\Phi_j = \{ \phi_j \}$ denote any set include arm $i$ as the selected set.
A proof to show this two statements are equivalent is provided in next Section (see Section \ref{pf:pf_fairness2}).

\end{definition}
\textbf{\textit{In summary, the goal of a solution approach is to generate a stochastic policy that never prefers one action over another if the cumulative long-term reward of selecting the latter one is higher.}}

\section{Proofs}
\begin{lemma}\label{lem:value_decay}
Consider the single arm $i$ with a finite horizon $T$, let $V_{m,i,t}(s_i)$ denote the value function start from time step $t\in[T]$ under the state $s_i$, we can have $V_{m,i, t}(s_i)>V_{m,i,t+1}(s_i)\geq 0$, for $\forall s_i\in \{0,1\}$.
\end{lemma}
\paragraph{Proof of Lemma~\ref{lem:value_decay}}
\label{subsec:proof_lem_value_decay}
\begin{proof}
We drop the subscript $i$, i.e., $V_t(s)=V_{i,t}(s_i)$. This is easy to prove. For state $s\in \{0,1 \}$, we can always find a algorithm that ensures $V_{m,t}(s)>V_{m,t+1}(s)$. For example, we assume the optimal algorithm for the state $s$ start from the time step $t+1$ is $\pi$, we can always find a algorithm $\pi^{\prime}$: keep the same actions as the algorithm $\pi$ until reach the last time step $t=T$ as $V_{m,t}(s)$ will has one more time slot compared to $V_{m,t+1}(s)$, and then we pick the action for the last time step $T$ according to the observed state $s^{\prime}$. Since the reward is either $0$ or $1$, thus $V_{T}(s)\geq 0$,so we can have
\begin{equation}
V_{m,t;\pi^{\prime}}(s) = V_{m,t+1;\pi}(s) + \gamma^{T-t} V_{T}(s^\prime)\geq V_{m,t+1;\pi}(s).
\end{equation}
\end{proof}

\paragraph{Proof of Theorem~\ref{thm:index_decay}}
\begin{proof}
Consider the discount reward criterion with discount factor of $\gamma\in [0,1]$ (where $\gamma =1 $ corresponds to the average criterion). Again, we drop the subscript $i$ and let: $Q_{i,t}(s_i, a_i) = Q_t(s,a)$. Because the state $s\in\{0,1 \}$ is fully observable, We can easily calculate $m_T$, where it needs to satisfy $Q_{T}(s,a=0)=Q_{T}(s,a=1)$, i.e., $m_T + P_{s,1}^0= P_{s,1}^1$, thus $m_T = P_{s,1}^1-P_{s,1}^0$. Similarly, $m_{T-1}$ can be solved by assuming equation $Q_{T-1}(s,a=0)=Q_{T-1}(s,a=1)$:
\begin{equation}
\begin{aligned}
        P_{s,1}^0 + m_{T-1} +\gamma (P_{s,1}^0 V_{m_T,T}(1)
        &+ P_{s,0}^0V_{m_T,T}(0))= \\P_{s,1}^1
        &+\gamma (P_{s,1}^1 V_{m_T,T}(1)+P_{s,0}^1 V_{m_T,T}(0))\\
        \rightarrow m_{T-1} =(P_{s,1}^1-P_{s,1}^0) +\gamma(V_{m_T,T}(1)&(P_{s,1}^1-P_{s,1}^0)+V_{m_T,T}(0)(P_{s,0}^1-P_{s,0}^0))
\end{aligned}
\end{equation}
Because $P_{1,1}^1-P_{1,1}^0>0$ and $P_{0,1}^1-P_{0,1}^0>0$ from the structural constraint we mentioned before and $V_{m_T,T}(s)\geq 0$ according to Lemma~\ref{lem:value_decay},
we have $m_{T-1}>m_T=P_{s,1}^1-P_{s,1}^0$. Now we show $m_t>m_{t+1}$. Equivalently, this can be expressed as $\frac{\partial m_t}{\partial t}<0$. Because the state is fully observable, we first get the close form of $m_t$.
\begin{itemize}[leftmargin=*]
    \item{\verb|Case 1:|} The state $s=0$,
    \begin{equation}
    \begin{aligned}
        P_{0,1}^0 + m_t +\gamma (P_{0,0}^0 V_{m_{t+1},t+1}(0)+P_{0,1}^0 &V_{m_{t+1},t+1}(1))\\
        = P_{0,1}^1+\gamma (P_{0,0}^1 &V_{m_{t+1},t+1}(0)+P_{0,1}^1 V_{m_{t+1},t+1}(1))\\
        \rightarrow m_t =(P_{0,1}^1-P_{0,1}^0)+\gamma(V_{t+1}(0)&(P_{0,0}^1-P_{0,0}^0)+V_{t+1}(1)(P_{0,1}^1-P_{0,1}^0)).
\end{aligned}
\end{equation}
Intuitively, we can have $V_{m_{t},t}(0)>V_{m_{t+1},t+1}(0)$ (see Lemma~\ref{lem:value_decay}), and  $V_{m_{t},t}(1)>V_{m_{t+1},t+1}(1)$, we obtain $\frac{\partial V_{m_{t},t}(0)}{\partial t}<0$ and $\frac{\partial V_{m_{t},t}(1)}{\partial t}<0$. Hence, we can get
\begin{equation}
\frac{\partial m_t}{\partial t}
= \gamma\left((P_{0,0}^1-P_{0,0}^0)\frac{\partial V_{m_{t+1},t+1}(0)}{\partial t}+(P_{0,1}^1-P_{0,1}^0)\frac{\partial V_{m_{t+1},t+1}(1)}{\partial t}\right) <0
\end{equation}
    \item{\verb|Case 2:|} For state $s=1$, similarly, we can get
\begin{equation}
\frac{\partial m_t}{\partial t}
= \gamma\left((P_{1,0}^1-P_{1,0}^0)\frac{\partial V_{m_{t+1},t+1}(0)}{\partial t}+(P_{1,1}^1-P_{1,1}^0)\frac{\partial V_{m_{t+1},t+1}(1)}{\partial t})\right) <0
\end{equation}
\end{itemize}
Thus $\forall t <T: m_t>m_{t+1}\geq m_T = P_{s,1}^1-P_{s,1}^0$.
\end{proof}

\paragraph{Proof of Theorem~\ref{thm:lambda}}
\begin{proof}\label{pf:thm:index_decay}
According to the Equation~\ref{eq:Q(s,a)}, we have
\begin{equation}
\begin{aligned}
  \pi_i(s_i,a_i) &= e^{{Q}_i(s_i,a_i)-V_i(s_i)} = \frac{e^{Q_i(s_i,a)}}{e^{V_i(s_i)}}\\
\end{aligned}
\end{equation}
By replacing this into Equation~\ref{eq:lambda}, we can get
\begin{equation}
\begin{aligned}
    \lambda_i &= \log \pi_i(s_i,a_i=1) - \log \pi_i(s_i,a_i=0)\\
    &=\log\left( \frac{e^{Q_i(s_i,a_i=1)}}{e^{V_i(s_i)}} \right) - \log\left( \frac{e^{Q_i(s_i,a_i=0)}}{e^{V_i(s_i)}}\right) \\
    &=\log \left( \frac{e^{Q_i(s_i,a_i=1)}}{  e^{Q_i(s_i,a_i=0)}} \right)\\
    & = Q_i(s_i, a_i=1) - Q_i(s_i, a_i=0).
\end{aligned}
\end{equation}
Because as $c$ approaches infinity, our algorithm becomes deterministically selecting the arm with the highest value of $\lambda$. Let set $\phi^{\ast}$ to be the set of actions containing the $k$ arms with the highest-ranking of $\lambda$ value, and any $k$ arms that aren't among the top k are included in the set $\phi^{\prime}$.
Let $\phi^{-,\ast}$ and $\phi^{-,\prime}$ denote the set that includes all of the arms except those in set $\phi^{\ast}$ and $\phi^{\prime}$, respectively. Thus the first action vector can be represented as $\mathbf{a}=\mathbb{I}_{\{\phi\}}$, and the latter action vector is $\mathbf{a}^\prime=\mathbb{I}_{\{\phi^{\prime}\}}$. We could have:
\begin{equation}
    \begin{aligned}
    \sum_{i \in \phi^{\ast}} \lambda_i &\geq \sum_{j \in \phi^{\prime}} \lambda_j \\
    \sum_{i \in \phi^{\ast}} \left[ Q(s_i, a_i=1)-Q(s_i, a_i=0)\right] &\geq \sum_{j \in \phi^{\prime}} \left[ Q(s_j, a_i=1)-Q(s_j, a_i=0)\right]\\
    \sum_{i \in \phi^{\ast}} Q(s_i, a_i=1) + \sum_{j \in \phi^{\prime}} Q(s_j, a_i=0) &\geq \sum_{j \in \phi^{\prime}} Q(s_j, a_i=1) + \sum_{i \in \phi^{\ast}} Q(s_i, a_i=0)
    \end{aligned}
\end{equation}
Adding $\underset{z\notin \phi^{\ast} \wedge z\notin \phi^{\prime}}{\sum} Q(s_z, a_i=0)$  on both sides, we can have,
\begin{equation}
    \begin{aligned}
    \sum_{i\in \phi^{\ast}} Q(s_i, a_i=1) + \sum_{j\in \phi^{-,\ast}} Q(s_j, a_i=0) &\geq \sum_{i\in \phi^{\prime}} Q(s_i, a_i=1)
    + \sum_{j\in \phi^{-,\prime}} Q(s_j, a_i=0)\\
    \rightarrow Q(\mathbf{s},\mathbf{a}=\mathbb{I}_{\{\phi\}})&\geq Q(\mathbf{s},\mathbf{a}^\prime=\mathbb{I}_{\{\phi^{\prime}\}})
    \end{aligned}
\end{equation}
Thus from this we can see that selecting action in the action set $\phi^{\ast}$ according to the $\lambda$ value can maximize the cumulative long-term reward, which lead to the optimal state-action value function $Q^\ast(\mathbf{s}, \mathbf{a})$ as well as the optimal value function, $V^\ast(\mathbf{s})$
\end{proof}

\paragraph{Proof of Theorem~\ref{thm:soft_fair}}
\begin{proof}
According to the Equation.~\ref{eq:lambda}, the probability of choosing an arm is the \textit{softmax} function on $\lambda$, which can guarantee that the higher the value of $\lambda$, the higher the probability of selecting that arm.
We can write down the probability of $\pi(\mathbf{s},\mathbf{a}=\mathbb{I}_{\{\phi\}})$, where $\phi$ is the set of selected arms when $k\neq 1$. More specifically, $\pi(\mathbf{s},\mathbf{a}=\mathbb{I}_{\{\phi\}}) = \Pi_{i\in \phi}\Pr(a_i=1|\mathbf{s})$. Note that $\Pr(a_i=1|\mathbf{s}) = softmax_c(c\cdot \lambda_i)$ if $k=1$. Intuitively, when $k\neq 1$, $\Pr(a_i=1|\mathbf{s})$ can be obtained through the brute-force permutation iteration over $\pi(\mathbf{s},\mathbf{a}=\mathbb{I}_{\{i\}})= softmax_c(c\cdot \lambda_i)$. It is easy to conclude that when $k\neq 1$, the higher the value of $\lambda_i$, the higher probability $\Pr(a_i=1|\mathbf{s})$. Formally, this is equivalent to:
\begin{equation}\label{eq:lambda_i_ineq_1}
    \pi(\mathbf{s},\mathbf{a}=\mathbb{I}_{\{\phi\}}) \geq \pi(\mathbf{s},\mathbf{a}=\mathbb{I}_{\{\phi^\prime \}}) \text{ if and only if } \sum_{i \in \phi} \lambda_i \geq \sum_{j \in \phi^\prime} \lambda_j.
\end{equation}
Similar to the proof of Theorem~\ref{pf:thm:index_decay}, we can have:
\begin{equation}\label{eq:lambda_i_ineq_2}
     \sum_{i \in \phi} \lambda_i \geq \sum_{j \in \phi^\prime} \lambda_j\text{ if and only if } Q(\mathbf{s},\mathbf{a}=\mathbb{I}_{\{\phi\}})\geq Q(\mathbf{s},\mathbf{a}^\prime=\mathbb{I}_{\{\phi^{\prime}\}}).
\end{equation}
Thus we have $\pi(\mathbf{s},\mathbf{a}=\mathbb{I}_{\{\phi\}}) \geq \pi(\mathbf{s},\mathbf{a}=\mathbb{I}_{\{\phi^\prime \}})$ if $Q(\mathbf{s},\mathbf{a}=\mathbb{I}_{\{\phi\}})\geq Q(\mathbf{s},\mathbf{a}^\prime=\mathbb{I}_{\{\phi^{\prime}\}})$. Our \textit{SoftFair} is fair under our proposed fairness constraint.

The trade-off is governed by $c$, where a larger $c$ means \textit{SoftFair} tends to choose arms with higher value, while a small $c$ means \textit{SoftFair} tends to ensure fairness among arms. More specifically, we have shown that selecting the top $k$ arms according to the $\lambda$ value at each time step $t$ when $c$ approaches infinity is equivalent to maximizing the cumulative long-term reward (Theorem~\ref{thm:lambda}). When $c$ is close to 0, the difference between $c\cdot\lambda$ is small and \textit{SoftFair} tends to uniformly sample $k$ arms. Therefore \textit{SoftFair} remains fair under our proposed fairness constraints, and the trade-off between fairness and optimal value is controlled by the multiplier parameter $c$.
\end{proof}

\paragraph{Proof of Definition~\ref{eq:fairness2}} \label{pf:pf_fairness2}
We rewrite the two statements that are equal:
\begin{equation}
\begin{aligned}
    \pi_t(\mathbf{s},\mathbf{a}) \geq \pi_t(\mathbf{s}, \mathbf{a}^\prime) &\text{ if and only if } Q^\ast(\mathbf{s},\mathbf{a})\geq Q^\ast(\mathbf{s},\mathbf{a}^\prime)\\
    \lambda_i \geq \lambda_j &\text{ if and only if } \sum_{\phi_i: \phi_i \in \Phi_i}\pi_t(\mathbf{s},\mathbf{a} = \mathbb{I}_{\{ \phi_i \}}) \geq \sum_{\phi_j: \phi_j \in \Phi_j}\pi_t(\mathbf{s},\mathbf{a} = \mathbb{I}_{\{ \phi_j \}})
\end{aligned}
\end{equation}
\begin{proof}
For any set of selected arms, let $\phi_i$ denote a set that always include arm $i$, and $\phi_j$ denote a set that always include arm $j$, and $\phi_{i,j}$ denote a set that always include arm $i$ and arm $j$. And $\Phi_i$ denote the set of $\phi_i$, formally, we have $\Phi_i = \{ \phi_i\}$ and $\Phi_j = \{\phi_j \}$ and $\Phi_{i,j}=\{ \phi_{i,j}\}$.
Similarly, let $\phi_{i,-j}$ denote the a that always include arm $i$ and but not include arm $j$, and $\phi_{j,i}$ denote a set that always include arm $j$ and but not include arm $i$. Thus we have:
\begin{equation}
\begin{aligned}
    \phi_i = \phi_{i,-j} + \phi_{i,j} &\text{ and } \phi_j = \phi_{j,-i} + \phi_{j,i} \\
    \Phi_i = \Phi_{j,-i} + \Phi_{i,j} &\text{ and } \Phi_j = \Phi_{j,-i} + \Phi_{j,i}
\end{aligned}\label{eq:divide}
\end{equation}
We divide this into two terms:
\begin{itemize}
    \item 1st Term: $\Phi_{i,-j}$ and $\Phi_{j,-i}$. We can instead of consider a subset of $k-1$ selected arms which does not include arm $i$ and arm $j$, we denote a subset of arms as $\phi^\prime_{-i,-j} \in \Phi^\prime_{-i,-j}$. Thus $\phi_{i,-j}$ and $\phi_{j,-i}$ can be writing as:
    \begin{equation}
    \begin{aligned}
        \phi_{i,-j} = \{i\} + \phi^\prime_{-i,-j} &\text{ and } \phi_{j,-i} = \{j\} + \phi^\prime_{-i,-j}   \\
        \Phi_{i,-j} = \sum_{\phi^\prime_{-i,-j} \in \Phi^\prime_{-i,-j}}\left[ \{i\} + \phi^\prime_{-i,-j}\right] &\text{ and } \Phi_{j,-i} =\sum_{\phi^\prime_{-i,-j} \in \Phi^\prime_{-i,-j}}\left[ \{j\} + \phi^\prime_{-i,-j} \right]
    \end{aligned}\label{eq:divide_1}
    \end{equation}
    In this case, if $\lambda_i \geq \lambda_j$, we add $\sum_{h\in \phi^\prime_{-i,-j}}\lambda_h$ on both side, which we get:
\begin{equation}
    \begin{aligned}
    \lambda_i &\geq \lambda_j\\
    \lambda_i + \sum_{h\in \phi^\prime_{-i,-j}}\lambda_h &\geq \lambda_j + \sum_{h\in \phi^\prime_{-i,-j}}\lambda_h\\
    \sum_{i \in\phi_{i,-j}}\lambda_i &\geq \sum_{j \in \phi_{j,-i}}\lambda_j\\
    \end{aligned}\label{eq:case_1}
\end{equation}
According to Theorem~\ref{thm:lambda} and Theorem~\ref{thm:soft_fair}, we can conclude that $\pi_t(\mathbf{s},\mathbf{a}=\mathbb{I}_{\{ \phi_{i,-j} \}}) \geq \pi_t(\mathbf{s}, \mathbf{a}^\prime = \mathbb{I}_{\{ \phi_{j,-i} \}})$. This is equal to the first statement that $ \pi_t(\mathbf{s},\mathbf{a}) \geq \pi_t(\mathbf{s}, \mathbf{a}^\prime) \text{ if and only if } Q^\ast(\mathbf{s},\mathbf{a})\geq Q^\ast(\mathbf{s},\mathbf{a}^\prime)$.

Furthermore, because for $\forall \phi^\prime_{-i,-j} \in \Phi^\prime_{-i,-j}$, we have $\lambda_i + \sum_{h\in \phi^\prime_{-i,-j}}\lambda_h \geq \lambda_j + \sum_{h\in \phi^\prime_{-i,-j}}\lambda_h$, and by summation over second line of Equation~\ref{eq:case_1}
like Equation~\ref{eq:divide_1}, we have:
\begin{equation}\label{eq:term1}
    \begin{aligned}
    \lambda_i + \sum_{h\in \phi^\prime_{-i,-j}}\lambda_h &\geq \lambda_j + \sum_{h\in \phi^\prime_{-i,-j}}\lambda_h\\
    \sum_{\phi^\prime_{-i,-j} \in \Phi^\prime_{-i,-j}}  \left\{ \lambda_i + \sum_{h\in \phi_{-i,-j}}\lambda_h \right\}&\geq \sum_{\phi^\prime_{-i,-j} \in \Phi^\prime_{-i,-j}} \left\{ \lambda_j + \sum_{h\in \phi^\prime_{-i,-j}}\lambda_h \right\}\\
    \sum_{\phi_{i,-j} \in \Phi_{i,-j}} \left\{ \sum_{i \in\phi_{i,-j}}\lambda_i \right\}&\geq \sum_{\phi_{j,-i} \in \Phi_{j,-i}} \left\{ \sum_{j \in \phi_{j,-i}}\lambda_j \right\} \text{ (according to Eq.~\ref{eq:divide_1})}\\
    \rightarrow  \sum_{\phi_{i,-j} \in \Phi_{i,-j}}  \left\{ \pi_t(\mathbf{s},\mathbf{a}=\mathbb{I}_{\{ \phi_{i,-j} \}}) \right\}&\geq \sum_{\phi_{j,-i} \in \Phi_{j,-i}} \left\{ \pi_t(\mathbf{s}, \mathbf{a}^\prime = \mathbb{I}_{\{ \phi_{j,-i} \}}) \right\}\\
    \end{aligned}
\end{equation}

\item 2nd Term: $\Phi_{i,j}$ and $\Phi_{j, i}$. Because we can easily see that $\phi_{i,j}=\phi_{j, i}$ and $\Phi_{i,j}=\Phi_{j, i}$, thus we have:
\begin{equation}\label{eq:term2}
    \sum_{\phi_{i,j} \in \Phi_{i,j}}\pi_t(\mathbf{s},\mathbf{a} = \mathbb{I}_{\{ \phi_{i,j} \}}) = \sum_{\phi_{j,i} \in \Phi_{j, i}}\pi_t(\mathbf{s},\mathbf{a} = \mathbb{I}_{\{ \phi_{j, i} \}}).
\end{equation}
\end{itemize}
Overall, according to Equation~\ref{eq:divide}, by summing this two terms (Equation~\ref{eq:term1} and Equation~~\ref{eq:term2}), we have $\lambda_i \geq \lambda_j \text{ if and only if } \sum_{\phi_i: \phi_i \in \Phi_i}\pi_t(\mathbf{s},\mathbf{a} = \mathbb{I}_{\{ \phi_i \}}) \geq \sum_{\phi_j: \phi_j \in \Phi_j}\pi_t(\mathbf{s},\mathbf{a} = \mathbb{I}_{\{ \phi_j \}}) $. And this is equal to the first statement that $ \pi_t(\mathbf{s},\mathbf{a}) \geq \pi_t(\mathbf{s}, \mathbf{a}^\prime) \text{ if and only if } Q^\ast(\mathbf{s},\mathbf{a})\geq Q^\ast(\mathbf{s},\mathbf{a}^\prime)$.

\end{proof}

\paragraph{Proof of Lemma ~\ref{lem:lem}}
\begin{proof}
The upper bound can be obtained by showing that $\forall (\mathbf{s},\mathbf{a})$, state-action value at the $ep-$th iteration are bounded. More specifically,
\begin{equation}\label{eq:upper_bound}
    Q^{ep}(\mathbf{s},\mathbf{a})=\sum_{i=1}^n Q^{ep}_i (s_i,a_i)\leq n \cdot \underset{i}{\sup}Q^{ep}_i (s_i,a_i) = n\sum_{t=0}^{T} \gamma^t R_{max}
\end{equation}
We can prove this Equation~\ref{eq:upper_bound} through induction as follows,
When $t=1$, we start from the definition of our \textit{SoftFair} in Equation~\ref{eq:value} to have
\begin{equation}
    \begin{aligned}
    Q^1(\mathbf{s},\mathbf{a}) &= \Psi_{soft}Q^0(\mathbf{s},\mathbf{a})\\
    &= \underset{\mathbf{a}}{\sum}\Pr(\mathbf{a}|\mathbf{s})\underset{\mathbf{s}^\prime_{t+1}}{\sum} \Pr(\mathbf{s}^\prime_{t+1}|\mathbf{s},\mathbf{a})(R(\mathbf{s},\mathbf{a})+\gamma  V^0(\mathbf{s}^\prime))\\
    &= \underset{\mathbf{a}}{\sum}\Pr(\mathbf{a}|\mathbf{s})\underset{\mathbf{s}^\prime_{t+1}}{\sum} \Pr(\mathbf{s}^\prime_{t+1}|\mathbf{s},\mathbf{a})(R(\mathbf{s},\mathbf{a})+\gamma  \underset{\mathbf{a}^\prime}{\max} Q^0(\mathbf{s}^\prime,\mathbf{a}^\prime))\\
    &\leq R_{max} + \underset{\mathbf{s}^\prime_{t+1}}{\sum} \Pr(\mathbf{s}^\prime_{t+1}|\mathbf{s},\mathbf{a})(\gamma  \underset{\mathbf{a}^\prime}{\max} Q^0(\mathbf{s}^\prime,\mathbf{a}^\prime))\\
    &\leq R_{max} + \gamma \underset{\mathbf{s}^\prime_{t+1}}{\sum} \Pr(\mathbf{s}^\prime_{t+1}|\mathbf{s},\mathbf{a})R_{max}\\
    &=(1+\gamma)R_{max}
    \end{aligned}
\end{equation}
We then assume that $ep=K$, where $K>1$, $Q^{K}(\mathbf{s},\mathbf{a})\leq \sum_{t=0}^{K} \gamma^{t} R_{max}$ holds, then we have
\begin{equation}
    \begin{aligned}
    Q^{K+1}(\mathbf{s},\mathbf{a}) &= \Psi_{soft}Q^K(\mathbf{s},\mathbf{a})\\
    &= \underset{\mathbf{a}}{\sum}\Pr(\mathbf{a}|\mathbf{s})\underset{\mathbf{s}^\prime_{t+1}}{\sum} \Pr(\mathbf{s}^\prime_{t+1}|\mathbf{s},\mathbf{a})(R(\mathbf{s},\mathbf{a})+\gamma  V^K(\mathbf{s}^\prime))\\
    &\leq \underset{\mathbf{a}}{\sum}\Pr(\mathbf{a}|\mathbf{s})\underset{\mathbf{s}^\prime_{t+1}}{\sum} \Pr(\mathbf{s}^\prime_{t+1}|\mathbf{s},\mathbf{a})(R_{max}+\gamma  \underset{\mathbf{a}^\prime}{\max} Q^K(\mathbf{s}^\prime,\mathbf{a}^\prime))\\
    &\leq R_{max} + \gamma\underset{\mathbf{s}^\prime_{t+1}}{\sum} \Pr(\mathbf{s}^\prime_{t+1}|\mathbf{s},\mathbf{a})(  \underset{\mathbf{a}^\prime}{\max} Q^K(\mathbf{s}^\prime,\mathbf{a}^\prime))\\
    &\leq R_{max} + \gamma \underset{\mathbf{s}^\prime_{t+1}}{\sum} \Pr(\mathbf{s}^\prime_{t+1}|\mathbf{s},\mathbf{a}) \sum_{t=0}^{K} \gamma^t R_{max}\\
    &=\sum_{t=0}^{K+1} \gamma^{t} R_{max}
    \end{aligned}
\end{equation}
$R(\mathbf{s},\mathbf{a})\in[R_{min},R_{max}]$ where $R_{max} =n$ as we have $n$ arms, and $R_{min}=0$. Thus the upper bound is $\frac{n}{1-\gamma}$.
Similarly, we can prove the lower bound is 0.  we can conclude that the state-action value function is bounded within $[0,\frac{n}{1-\gamma}]$.
\end{proof}

\paragraph{Proof of Lemma~\ref{lem:bound}}
\begin{proof}
let $\mathbf{a}_{\{ i\}}$ denote the action to select the arm $i$ and remain other arms in passive, which is the shorthand for $\mathbf{a}=\mathbb{I}_{\{ i\}}$. We first sort ${Q(\mathbf{s},\mathbf{a}_{\{ i\}})}$ in the ascending order according to the $\lambda$ value. Assume we get ${Q(\mathbf{s},\mathbf{a}_{\{ 1\}^\prime}^\prime)}\geq \dots \geq {Q(\mathbf{s},\mathbf{a}_{\{ n\}^\prime}^\prime)}$ after sorting, and corresponding $\Pr(\cdot|\mathbf{s})$ becomes
$[\Pr(\mathbf{a}=\mathbb{I}_{\{1\}^\prime}|\mathbf{s}), \dots,\Pr(\mathbf{a}=\mathbb{I}_{\{n\}^\prime}|\mathbf{s})]^\top$.
According to Equation~\ref{eq:lambda}, when $k=1$, we have $\Pr(a_i=1|\mathbf{s})=\textit{softmax}_c ( c\cdot \lambda_i)$

Then $\forall Q$ and $\forall s$, we can get
\begin{equation}\label{eq:bound}
\begin{aligned}
    &\underset{\mathbf{a}}{\max} Q(\mathbf{s},\mathbf{a}) - (\Pr(\cdot|\mathbf{s}))^\top Q(\mathbf{s},\cdot)=\underset{\mathbf{a}}{\max} Q(\mathbf{s},\mathbf{a}) - \textit{softmax}_c^\top(c\cdot \lambda_{i^\prime})) Q(\mathbf{s},\cdot) \\
    =&Q(\mathbf{s},\mathbf{a}_{\{ 1\}^\prime}^\prime)) - \frac{\sum_{i=1}^n \exp[c\cdot \lambda_{i^\prime}]\cdot Q(\mathbf{s},\mathbf{a}_{\{ i\}^\prime}^\prime)}{\sum_{i=1}^n \exp[c\cdot \lambda_{i^\prime}]}\\
    =&\frac{\sum_{i=1}^n \exp[c\cdot \lambda_{i^\prime}]\cdot [Q(\mathbf{s},\mathbf{a}_{\{ 1\}^\prime}^\prime)-Q(\mathbf{s},\mathbf{a}_{\{ i\}^\prime}^\prime)]}{\sum_{i=1}^n \exp[c\cdot \lambda_{i^\prime}]}
\end{aligned}
\end{equation}

According to Equation~\ref{eq:Q(s,a)}, we can get $\lambda_{i^\prime} = Q(s_{i^\prime}, a_{i^\prime}=1)-Q(s_{i^\prime}, a_{i^\prime}=0)$.
Let $K=\exp [\sum_{j=1}^n Q(s_j, a_j=0)]$, and we have
\begin{equation}\label{eq:bound_3}
    K\cdot \exp[c\cdot \lambda_{i^\prime}] =\exp[ Q(s_{i^\prime}, a_{i^\prime}=1) + \sum_{j \in \phi ^{-.i^\prime} }Q(s_j, a_j=0)] = \exp[ Q(\mathbf{s},\mathbf{a}_{\{ i\}^\prime}^\prime)].
\end{equation}
Here $\phi ^{-.i^\prime}$ denote the set that include all of the arms except arm $i^\prime$. Thus by applying Equation~\ref{eq:bound_3} to Equation~\ref{eq:bound}, we have
\begin{equation}\label{eq:bound_2}
\begin{aligned}
    &\frac{\sum_{i=1}^n \exp[c\cdot \lambda_{i^\prime}]\cdot [Q(\mathbf{s},\mathbf{a}_{\{ 1\}^\prime}^\prime)-Q(\mathbf{s},\mathbf{a}_{\{ i\}^\prime}^\prime)]}{\sum_{i=1}^n \exp[c\cdot \lambda_{i^\prime}]} \\
    =&\frac{ K \cdot \sum_{i=1}^n \exp[c\cdot \lambda_{i^\prime}]\cdot [Q(\mathbf{s},\mathbf{a}_{\{ 1\}^\prime}^\prime)-Q(\mathbf{s},\mathbf{a}_{\{ i\}^\prime}^\prime)]}{K \cdot \sum_{i=1}^n \exp[c\cdot \lambda_{i^\prime}]} \\
    =&\frac{\sum_{i=1}^n \exp[c\cdot Q(\mathbf{s},\mathbf{a}_{\{ i\}^\prime}^\prime)]\cdot[Q(\mathbf{s},\mathbf{a}_{\{ 1\}^\prime}^\prime)-Q(\mathbf{s},\mathbf{a}_{\{ i\}^\prime}^\prime)]}{\sum_{i=1}^n \exp[c\cdot Q(\mathbf{s},\mathbf{a}_{\{ i\}^\prime}^\prime)]}\\
\end{aligned}
\end{equation}

Let $\delta_{\{ i\}^\prime}(\mathbf{s})=Q(\mathbf{s},\mathbf{a}_{\{ 1\}^\prime}^\prime)-Q(\mathbf{s},\mathbf{a}_{\{ i\}^\prime}^\prime)$, and we have $\delta_{\{ i\}^\prime}(\mathbf{s})\geq 0$, and $\delta_{\{ 1\}^\prime}(\mathbf{s})=0$.
Using Equation~\ref{eq:bound_2}, we have

\begin{equation}\label{eq:bound2}
    \begin{aligned}
    &\frac{\sum_{i=1}^n \exp[c\cdot Q(\mathbf{s},\mathbf{a}_{\{ i\}^\prime}^\prime)]\cdot[Q(\mathbf{s},\mathbf{a}_{\{ 1\}^\prime}^\prime)-Q(\mathbf{s},\mathbf{a}_{\{ i\}^\prime}^\prime)]}{\sum_{i=1}^n \exp[c\cdot Q(\mathbf{s},\mathbf{a}_{\{ i\}^\prime}^\prime)]}\\
    =&\frac{\sum_{i=1}^n \exp \left[c\cdot \left(Q(\mathbf{s},\mathbf{a}_{\{ 1\}^\prime}^\prime)-\delta_{\{ i\}^\prime}(\mathbf{s})\right)\right]\cdot\delta_{\{ i\}^\prime}(\mathbf{s})}{\sum_{i=1}^n \exp \left[c\cdot \left(Q(\mathbf{s},\mathbf{a}_{\{ 1\}^\prime}^\prime)-\delta_{\{ i\}^\prime}(\mathbf{s})\right)\right]}\\
    =&\frac{\sum_{i=1}^n \exp[-c \cdot \delta_{\{ i\}^\prime}(\mathbf{s})]\cdot \delta_{\{ i\}^\prime}(\mathbf{s})}{\sum_{i=1}^{n}\exp[-c \cdot \delta_{\{ i\}^\prime}(\mathbf{s})]}\\
    =&\frac{\sum_{i=2}^n \exp[-c \cdot \delta_{\{ i\}^\prime}(\mathbf{s})]\cdot \delta_{\{i \}^\prime}(\mathbf{s})}{1+\sum_{i=2}^{n}\exp[-c \cdot \delta_{\{ i\}^\prime}(\mathbf{s})]}
    \end{aligned}
\end{equation}
Now from equation~\ref{eq:bound2}, we can derive the upper bound.
We follow the work in~\cite{song2019revisiting}, we take advantage of the fact that for any two non-negative sequences $\{ x_i\}$ and $\{y_i \}$,
\begin{equation}\label{eq:fact}
    \frac{\sum_i x_i}{1+\sum_i y_i}\leq \sum_i \frac{x_i}{1+y_i}
\end{equation}
Apply Equation~\ref{eq:fact} to Equation~\ref{eq:bound2}
\begin{equation}
    \begin{aligned}
    \frac{\sum_{i=2}^n \exp[-c \cdot \delta_{\{i \}^\prime}(\mathbf{s})]\cdot \delta_{\{i \}^\prime}(\mathbf{s})}{1+\sum_{i=2}^{n}\exp[-c \cdot \delta_{\{i \}^\prime}(\mathbf{s})]}&\leq \sum_{i=2}^n \frac{\exp[-c \cdot \delta_{\{i \}^\prime}(\mathbf{s})]\cdot \delta_{\{i \}^\prime}(\mathbf{s})}{1+\exp[-c \cdot \delta_{\{i \}^\prime}(\mathbf{s})]}\\
    &=\sum_{i=2}^n\frac{\delta_{\{i \}^\prime}(\mathbf{s})}{1+\exp[c\cdot \delta_{\{i \}^\prime}(\mathbf{s})]}
    \end{aligned}
\end{equation}\label{eq:bound3}
Intuitively, we have $0\leq \delta_{\{i \}^\prime}(\mathbf{s})\leq 1$, thus through using Taylor series, we can rewrite Equation~\ref{eq:bound3} as
\begin{equation}
\begin{aligned}
    &\sum_{i=2}^n\frac{\delta_{\{i \}^\prime}(\mathbf{s})}{1+\exp[c\cdot \delta_{\{i \}^\prime}(\mathbf{s})]}\\
    =&\sum_{i=2}^n\frac{\delta_{\{i \}^\prime}(\mathbf{s})}{1+1+c\cdot \delta_{\{i \}^\prime}(\mathbf{s}) + 0.5c^2 \cdot \delta_{\{i \}^\prime}(\mathbf{s})^2+\dots}\\
    \leq&\sum_{i=2}^n \frac{1}{2+c} =\frac{n-1}{2+c}
\end{aligned}
\end{equation}
We can also get the lower bound as
\begin{equation}
\begin{aligned}
    &\frac{\sum_{i=2}^n \exp[-c \cdot \delta_{\{i \}^\prime}(\mathbf{s})]\cdot \delta_{\{i \}^\prime}(\mathbf{s})}{1+\sum_{i=2}^{n}\exp[-c \cdot \delta_{\{ \}^\prime}(\mathbf{s})]}\geq \frac{\sum_{i=2}^n \exp[-c \cdot \delta_{\{i \}^\prime}(\mathbf{s})]\cdot \delta_{\{i \}^\prime}(\mathbf{s})}{n}\\
    &=\frac{\sum_{i=2}^n \delta_{\{i \}^\prime}(\mathbf{s})}{n\exp[c \cdot \delta_{\{i \}^\prime}(\mathbf{s})]}\geq\frac{\sum_{i=2}^n  \delta_{\{ i\}^\prime}(\mathbf{s})}{n\exp[c \cdot \delta(\mathbf{s})]}
    \geq\frac{\delta(\mathbf{s})}{n\exp[c \cdot \delta(\mathbf{s})]}
\end{aligned}
\end{equation}

\end{proof}

\paragraph{Proof of Theorem~\ref{thm:soft_bound}}
\begin{proof}
The proof is similar to the work by~\citet{song2019revisiting} but modified for our RMAB setting.
\paragraph{Upper bound}: We derive the upper bound through induction. We start from the definition for $\Psi$ and $\Psi_{soft}$ in Equation~\ref{eq:Psi} and Equation~\ref{eq:Psi_soft}. When $ep=1$, we have
\begin{equation}
    \begin{aligned}
    &\Psi Q^{0}(\mathbf{s},\mathbf{a})-\Psi_{soft} Q^{0}(\mathbf{s},\mathbf{a})\\
    =&\gamma\underset{\mathbf{s}^\prime}{\sum} \Pr(\mathbf{s}^\prime|\mathbf{s},\mathbf{a})[\underset{\mathbf{a}^\prime}{\max} Q^{0}(\mathbf{s}^\prime,\mathbf{a}^\prime) - Pr(\mathbf{a}^\prime|s^\prime)Q^{0}(\mathbf{s}^\prime,\mathbf{a}^\prime)]\\
    \geq& 0
    \end{aligned}
\end{equation}
Assume $\Psi Q^{K}(\mathbf{s},\mathbf{a})-\Psi_{soft} Q^{K}(\mathbf{s},\mathbf{a})\geq 0$ holds for $ep=K$ where $K>1$. When $ep=K+1$, we have
\begin{equation}
    \begin{aligned}
    &\Psi Q^{K+1}(\mathbf{s},\mathbf{a})-\Psi_{soft} Q^{K+1}(\mathbf{s},\mathbf{a})\\
    =&\Psi \Psi^K Q^{0}(\mathbf{s},\mathbf{a}) - \Psi_{soft} \Psi_{soft}^K Q^{0}(\mathbf{s},\mathbf{a})\\
    \geq&\Psi \Psi^K_{soft} Q^{0}(\mathbf{s},\mathbf{a}) - \Psi_{soft} \Psi_{soft}^K Q^{0}(\mathbf{s},\mathbf{a})\\
    \geq &0
    \end{aligned}
\end{equation}

Since $\underset{ep\rightarrow \infty}{\lim}{\Psi^K}Q^{0}(\mathbf{s},\mathbf{a}) = Q^{\ast}(\mathbf{s},\mathbf{a})$, where $Q^{\ast}(\mathbf{s},\mathbf{a})$ is the optimal state action value, and thus $\underset{ep\rightarrow \infty}{\lim}{\Psi^K}V^{0}(\mathbf{s}) = V^{\ast}(\mathbf{s})$, where $V^{\ast}(\mathbf{s})$ is the optimal value function. Therefore, we have $\underset{ep\rightarrow \infty}{\lim \sup}{\Psi^K}Q^{0}(\mathbf{s},\mathbf{a})\leq Q^{\ast}(\mathbf{s},\mathbf{a})$ and $\underset{ep\rightarrow \infty}{\lim \sup}{\Psi^K}V^{0}(\mathbf{s})\leq V^{\ast}(\mathbf{s})$.

\paragraph{Lower bound}: Similarly, we derive the lower bound through induction from the definition for $\Psi$ and $\Psi_{soft}$ in Equation~\ref{eq:Psi} and Equation~\ref{eq:Psi_soft}. When $ep=1$, we have
\begin{equation}
    \begin{aligned}
    &\Psi Q^{0}(\mathbf{s},\mathbf{a})-\Psi_{soft} Q^{0}(\mathbf{s},\mathbf{a})\\
    =&\gamma\underset{\mathbf{s}^\prime}{\sum} \Pr(\mathbf{s}^\prime|\mathbf{s},\mathbf{a})[\underset{\mathbf{a}^\prime}{\max} Q^{0}(\mathbf{s}^\prime,\mathbf{a}^\prime) - Pr(\mathbf{a}^\prime|s^\prime)Q^{0}(\mathbf{s}^\prime,\mathbf{a}^\prime)]\\
    \leq& \gamma \underset{\mathbf{s}^\prime}{\sum} \Pr(\mathbf{s}^\prime|\mathbf{s},\mathbf{a}) \eta = \gamma \eta = \gamma \frac{n-1}{2+c}
    \end{aligned}
\end{equation}
where $\eta$ is the upper bound of $\underset{\mathbf{a}^\prime}{\max} Q^{0}(\mathbf{s}^\prime,\mathbf{a}^\prime) - Pr(\mathbf{a}^\prime|\mathbf{s}^\prime)Q^{0}(\mathbf{s}^\prime,\mathbf{a}^\prime)$, can be written as $\eta = \underset{\mathbf{s}^\prime}{\sup}[\underset{\mathbf{a}^\prime}{\max} Q^{0}(\mathbf{s}^\prime,\mathbf{a}^\prime) - Pr(\mathbf{a}^\prime|\mathbf{s}^\prime)Q^{0}(\mathbf{s}^\prime,\mathbf{a}^\prime)]$. According to Lemma~\ref{lem:bound}, we have $\eta = \frac{n-1}{2+c}$.
Then we assume $\Psi Q^{K}(\mathbf{s},\mathbf{a})-\Psi_{soft} Q^{K}(\mathbf{s},\mathbf{a})\leq \sum_{i=1}^K \gamma^i \frac{n-1}{2+c}$ holds for $ep=K$ where $K>1$. When $ep=K+1$, we have
\begin{equation}
    \begin{aligned}
    &\Psi Q^{K+1}(\mathbf{s},\mathbf{a})-\Psi_{soft} Q^{K+1}(\mathbf{s},\mathbf{a})\\
    =&\Psi \Psi^K Q^{0}(\mathbf{s},\mathbf{a}) - \Psi_{soft} \Psi_{soft}^K Q^{0}(\mathbf{s},\mathbf{a})\\
    \leq & \Psi\left( \Psi_{soft}^K Q^{0}(\mathbf{s},\mathbf{a}) +  \sum_{i=1}^K \gamma^i \frac{n-1}{2+c} \right) - \Psi_{soft} \Psi_{soft}^K Q^{0}(\mathbf{s},\mathbf{a})\\
    =&\sum_{i=2}^{K+1} \gamma^i \frac{n-1}{2+c} + (\Psi-\Psi_{soft}) \Psi_{soft}^K Q^{0}(\mathbf{s},\mathbf{a})\\
    \leq &\sum_{i=2}^{K+1} \gamma^i \frac{n-1}{2+c} + \gamma^i \frac{n-1}{2+c}\\
    =&\sum_{i=1}^{K+1} \gamma^i \frac{n-1}{2+c}
    \end{aligned}
\end{equation}
Thus we have
\begin{equation}
    \underset{K\rightarrow \infty}{\lim}\Psi Q^{K}(\mathbf{s},\mathbf{a})-\Psi_{soft} Q^{K}(\mathbf{s},\mathbf{a})\leq \underset{K\rightarrow \infty}{\lim} \sum_{i=1}^{K} \gamma^i \frac{n-1}{2+c} = \frac{n-1}{(2+c)(1-\gamma)}
\end{equation}

\end{proof}

\section{Datasets}
\label{app_sec:dataset}
\paragraph{Realistic dataset}: Obstructive sleep apnea (OSA) is a common disorder, and CPAP is considered the gold standard of treatment, effectively addressing a range of adverse outcomes. However, effectiveness is often limited by non-adherence. CPAP adherence varies substantially across settings; 29\% to 83\% of patients enrolled in the study reported using CPAP recommended hours per night~\cite{weaver2008adherence}.

Our CPAP dataset is provided by~\citet{kang2013markov}, they define states as the CPAP usage levels (0: did not adhere, 1: adhere), and estimate transition probabilities between CPAP usage levels to build a Markov chain. They divide patients into two groups where patients in the first cluster exhibit ‘adherent’ behavior, they are most likely to transition to and remain in a good CPAP usage state, while the ‘non-adherent’ patient type we consider in the second cluster shows a weak trend to transition to good CPAP usage state.

We consider an intervention effect broadly characterizing supportive interventions such as text message reminders, telemonitoring and telephone support, which is associated with a random increase from 5-50\% to good CPAP use status per night for both two groups. We add a small noise to each transition matrix so that the dynamics of each individual arm in same group is different.
The initial state vector $\mathbf{s}$ is randomly assigned.

\paragraph{Synthetic dataset}: For the synthetic dataset, it is natural to require strict positive transition matrix entries, and in order to mimic the real-world setting, following~\cite{mate2020collapsing}, we impose four structural constraints: $(\romannumeral1)$ $P_{0,1}^0<P_{1,1}^0$, $(\romannumeral2)$ $P_{0,1}^1<P_{1,1}^1$, $(\romannumeral3)$ $P_{0,1}^0<P_{0,1}^1$ and $(\romannumeral4)$ $P_{1,1}^0<P_{1,1}^1$. Those constraints imply that arms are more likely to stay in a good state when there is positive intervention involved (active action) compared to no intervention (passive action).

\end{document}